\documentclass[12pt]{article}
\usepackage{amssymb,amsmath, amsthm}
\usepackage{framed}
\usepackage{mathrsfs}
\usepackage{enumerate}
\usepackage{tikz}
\usetikzlibrary{matrix}
\usepackage[all]{xy}
\usetikzlibrary{shapes}
\usepackage{hyperref}
\usepackage{multirow} 
\usepackage{bm}
\usepackage[algo2e,linesnumbered,vlined,ruled]{algorithm2e}
\usepackage{algpseudocode}
\usepackage{cite}
\usepackage{tikz}
\definecolor{gray}{rgb}{0.9, 0.9, 0.9}
\usepackage{enumitem}
\setenumerate[0]{leftmargin=*}

\newcommand{\inclu}[0] {\ar@{^{(}->}}

\newcommand{\RR}{\mathbb{R}}

\newcommand{\EE}{{\mathbb E}}

\newcommand{\cL}{\mathcal{L}}

\newcommand{\cO}{\mathcal{O}}

\newcommand{\cS}{\mathcal{S}}
\newcommand{\cA}{\mathcal{A}}
\newcommand{\cP}{\mathcal{P}}
\newcommand{\be}{\mathbf{e}}

\newcommand{\sign}{\mathrm{sign}}

\newcommand{\cN}{\mathcal{N}}

\newcommand{\cB}{\mathcal{B}}
\newcommand{\cF}{\mathcal{F}}
\newcommand{\cG}{\mathcal{G}}

\usepackage[margin=1.1in]{geometry}

\newcommand{\argmax}{\operatornamewithlimits{argmax}}

\SetAlgoSkip{bigskip}
\SetKwInput{Input}{Input\hspace{34pt}{ }}
\SetKwInput{Output}{Output\hspace{27pt}{ }}
\SetKwInput{Initialize}{Initialize\hspace{2pt}{ }}
\SetKwInput{Parameters}{Parameters\hspace{2pt}{ }}

\newtheorem{theorem}{Theorem}[section]
\newtheorem{proposition}[theorem]{Proposition}
\newtheorem{lemma}[theorem]{Lemma}

\newtheorem{assumption}[theorem]{Assumption}

\newtheorem{example}{Example}[section]

\usepackage{mathtools}

\title{On the Convergence and Sample Efficiency of Variance-Reduced Policy Gradient Method}
\author{\,\,\,\,\,\,\,\,\,\,
{Junyu Zhang} \thanks{Department of Electrical and Computer Engineering, Princeton University \& Department of Industrial Systems Engineering and Management, National University of Singapore, junyuz@princeton.edu }\and
{Chengzhuo Ni} \thanks{Department of Electrical and Computer Engineering, Princeton University, chengzhuo.ni@princeton.edu}\and
{Zheng Yu} \thanks{Department of Electrical and Computer Engineering, Princeton University, zhengy@princeton.edu}\,\,\,\,\,\,\,\,\,\,\and
{Csaba Szepesvari} \thanks{Department of Computer Science, University of Alberta \& Deepmind, szepesva@ualberta.ca}\and
{Mengdi Wang} \thanks{Department of Electrical and Computer Engineering, Princeton University \& Deepmind, mengdiw@princeton.edu}
}

\begin{document}
\date{}
\maketitle
\begin{abstract}
	Policy gradient (PG) gives rise to a rich class of reinforcement learning (RL) methods. Recently, there has been an emerging trend to accelerate the existing PG methods such as REINFORCE by the \emph{variance reduction} techniques.  However, all existing variance-reduced PG methods heavily rely on an uncheckable importance weight assumption made for every single iteration of the algorithms. In this paper, a simple gradient truncation mechanism is proposed to address this issue. Moreover, we design a Truncated Stochastic Incremental Variance-Reduced Policy Gradient (TSIVR-PG) method, which is able to maximize not only a cumulative sum of rewards but also a general utility function over a policy's long-term visiting distribution.  We show an $\tilde{\cO}(\epsilon^{-3})$ sample complexity for TSIVR-PG to find an $\epsilon$-stationary policy. By assuming the overparameterizaiton of policy and exploiting the hidden convexity of the problem, we further show that TSIVR-PG converges to global $\epsilon$-optimal policy with $\tilde{\cO}(\epsilon^{-2})$ samples. 
\end{abstract}
\section{Introduction}
In this paper, we investigate the theoretical properties of Policy Gradient (PG) methods for Reinforcement Learning (RL) \cite{sutton2018reinforcement}. In view of RL as a policy optimization problem, the PG method parameterizes the policy function and conduct gradient ascent search to improve the policy. In this paper, we consider the soft-max policy parameterization
\begin{equation}
	\label{defn:policy-para}
	\pi_{\theta}(a|s) = \frac{\exp\{\psi(s,a;\theta)\}}{\sum_{a'}\exp\{\psi(s,a';\theta)\}}
\end{equation}
where $(s,a)$ is a state-action pair and $\psi$ is some smooth function. Potentially, one can set the function $\psi$ to be some deep neural network with weights $\theta$ and input $(s,a)$. The main problem considered in this paper is the policy optimization for a \emph{general utility} function:
\begin{equation}
	\label{prob:main-0}
	\max_{\theta} R(\pi_\theta) := F(\lambda^{\pi_{\theta}}),
\end{equation}
where $F$ is a general smooth function, and $\lambda^{\pi_{\theta}}$ denotes the unnormalized state-action occupancy measure (also referred to as the visitation measure).  For any policy $\pi$ and initial state distribution $\xi$,
\begin{eqnarray}
	\label{defn:occupancy}
	\lambda^{\pi}(s,a) := \sum_{t=0}^{+\infty}\gamma^t\cdot \mathbb{P}\Big( s_t = s, a_t = a\,\big|\, \pi, s_0\sim\xi\Big).
\end{eqnarray}
When $F$ is linear, the problem reduces to the standard policy optimization problem where the objective is to maximize a cumulative sum of rewards. When $F$ is nonlinear, problem \eqref{prob:main-0} goes beyond standard Markov decision problems: examples include the max-entropy exploration \cite{hazan2019provably}, risk-sensitive RL \cite{zhang2020cautious}, certain set constrained RL  \cite{miryoosefi2019reinforcement}, and so on. 

In the standard cumulative-return case (i.e., $F$ is linear), numerous works have studied  PG methods in various scenarios, see e.g. \cite{williams1992simple,baxter2001infinite,zhao2012analysis,konda1999actor,konda2000actor,schulman2015trust,schulman2017proximal}. 
To the authors' best knowledge, a most recent variant of PG methods, using the SARAH/Spider stochastic variance reduction technique \cite{fang2018spider,nguyen2017sarah}, finds a local $\epsilon$-stationary policy using $\cO(\epsilon^{-3})$ samples \cite{xu2019sample,pham2020hybrid}. This poses a contrast with the known $\tilde{O}(\epsilon^{-2})$ sample complexity results that can be achieved by various value-based methods \cite{azar2013minimax,sidford2018variance,sidford2018near} and are provably matching information-theoretic lower bounds \cite{even2006action,azar2011reinforcement,azar2013minimax}. In this paper, we attempt to close this gap and prove an $\tilde{\cO}(\epsilon^{-2})$ sample complexity bound for a PG method. Most importantly, when it comes to PG estimation, the application of the variance reduction technique typically relies on certain off-policy PG estimator, resulting in the difficulty of distribution shift. We notice that none of the existing variance-reduced PG methods attempt to address this challenge. Instead, they directly make an uncheckable assumption that the variance of the importance weight is bounded for every policy pair encountered in running the algorithm, see e.g. \cite{papini2018stochastic,xu2019sample,xu2020improved,pham2020hybrid}. In this paper, we propose a simple gradient truncation mechanism to fix this issue.

Next, let us go beyond cumulative return and consider policy optimization for a general utility where $F$ may be nonlinear. However, much less is known in this setting.  The nonlinearity of $F$ invalidates the concept of Q-function and value function, leading to the failure of policy gradient theorem \cite{sutton2000policy}. To overcome such difficulty,
\cite{zhang2020variational} showed that the policy gradient for the general utilities is the solution to a min-max problem. However, estimating a single PG is highly nontrivial in this case. It is still unclear how to make PG methods to use samples in a most efficient way.

In this paper, we aim to investigate the convergence and sample efficiency of the PG method, using episodic sampling, for both linear $F$ (i.e., cumulative rewards) and nonlinear $F$ (i.e., general utility). Observe that problem \eqref{prob:main-0} is an instance of the Stochastic Composite Optimization (SCO) problem \cite{wang2017stochastic,wang2017accelerating}: 
$$\min_x f (\mathbb{E}_\nu[g_\nu(x)]),$$ which involves an inner expectation that corresponds to the occupancy measure $\lambda^{\pi}$. Motivated by this view point, we attempt to develop stochastic policy gradient method with provable finite-sample efficiency bounds.

\textbf{Main results.}  Our main results are summarized below.
\begin{itemize} 
	\item We propose the TSIVR-PG algorithm  to solve problem \eqref{prob:main-0} via episodic sampling. It provides a conceptually simple stochastic gradient approach for solving  general utility RL.
	\item We provide a gradient truncation mechanism to address the distribution shift difficulty in variance-reduced PG methods. Such difficulty has never been addressed in previous works.
	\item We show that TSIVR-PG finds an $\epsilon$-stationary policy using $\tilde{O}(\epsilon^{-3})$ samples if $F$ and $\psi$ are general smooth functions. When $F$ is concave and $\psi$ satisfies certain overparameterization condition, we show that TSIVR-PG obtains a  gloal $\epsilon$-optimal policy using $\tilde{O}(\epsilon^{-2})$ samples. 
\end{itemize}
\textbf{Technical contribution.} Our analysis technique is also of independent interest in the relating areas.
\begin{itemize}
	\item For stochastic composite optimization (SCO), most existing algorithms require estimating the Jacobian matrix of the inner mapping, which corresponds to $\nabla_\theta \lambda^{\pi_{\theta}}$ in our setting. This is in practice prohibitive if the Jacobian matrix has high dimensions, which is exactly the case in our problem. Unlike SCO algorithms such as \cite[etc.]{lian2017finite,zhang2019stochastic,zhang2019multi}, our analysis enables us to avoid the Jacobian matrix estimation.
	\item For the stochastic variance-reduced gradient methods, our analysis implies a convergence of SARAH/Spider methods to global optimality and a new $\cO(\epsilon^{-2})$ sample complexity for nonconvex problems with ``hidden convexity'' structure, which has not been studied in the optimization community yet.
\end{itemize} 
\section{Related Works}
\label{sec:related-works}
Policy gradient gives rises to a rich family of RL algorithms, such as REINFORCE and many of its variants \cite{williams1992simple,baxter2001infinite,zhao2012analysis}, as well as extensions such as the natural policy gradient methods \cite{kakade2001natural,peters2008natural}, the actor-critic methods \cite{konda1999actor,konda2000actor,mnih2016asynchronous}, the trust-region policy optimization \cite{schulman2015trust,zhao2019stochastic,shani2020adaptive}, and the proximal policy optimization method \cite{schulman2017proximal,liu2019neural}, etc.  In this paper we mainly focus on REINFORCE-type methods, where many of them need $\tilde{\cO}(\epsilon^{-4})$ samples to find an  $\epsilon$-stationary solution, including the vanilla REINFORCE \cite{williams1992simple}, as well as its variants with baseline \cite{zhao2012analysis,sutton2018reinforcement} and GPOMDP \cite{baxter2001infinite}, etc. By incorporating the stochastic variance reduction techniques, the sample efficiency of PG methods can be further improved. In \cite{papini2018stochastic}, the SVRG \cite{johnson2013accelerating} variance reduction scheme is adopted and an $\cO(\epsilon^{-4})$ sample complexity is achieved, which is later improved to $\cO(\epsilon^{-10/3})$ by \cite{xu2019sample}. With additional Hessian information, \cite{shen2019hessian} achieved an $\cO(\epsilon^{-3})$ complexity. By utilizing a more efficient SARAH/Spider \cite{nguyen2017sarah,fang2018spider} variance reduction scheme, people are able to achieve $\cO(\epsilon^{-3})$ sample complexity without second-order information \cite{xu2019sample,pham2020hybrid}. We would like to comment that these results are only for finding $\epsilon$-stationary (rather than near-optimal) solutions, and all of them requires an uncheckable condition on the importance weights in every iteration.

Recently, for cumulative reward, a series of works have started to study the convergence of policy gradient method to global optimal solutions \cite{agarwal2020optimality,fazel2018global,zhang2019global,bhandari2019global,mei2020global,bhandari2020note,cen2020fast,zhang2020sample}. 
In particular, \cite{zhang2020variational} exploited the hidden convexity property of the MDP problem and established the convergence to global optimality for general utility RL problem, as long as the policy gradient can be computed exactly. 

Our approach is related to the stochastic composite optimization (SCO) \cite{wang2017stochastic,wang2017accelerating}. 
For the general composition problem, there have been numerous developments, including momentum-based and multi-time-scale algorithms \cite{wang2017stochastic,wang2017accelerating,ghadimi2020single}, and various composite stochastic variance-reduced algorithms \cite{lian2017finite,huo2018accelerated,zhang2019composite,zhang2019stochastic}. Our approach is also inspired by variance reduction techniques that were initially used for stochastic convex optimization, see \cite{johnson2013accelerating,schmidt2017minimizing,defazio2014saga,nguyen2017sarah}; and were later on extended to the stochastic nonconvex optimization problems \cite{allen2016variance,reddi2016stochastic,j2016proximal,reddi2016fast,fang2018spider,nguyen2019finite}. In particular, we will utilize the SARAH/Spider scheme \cite{fang2018spider,nguyen2019finite}.

\section{Problem Formulation}
Consider an MDP with a general utility function, denoted as $\text{MDP}(\cS,\cA,\cP,\gamma,F)$, where $\cS$ is a finite state space, $\cA$ is a finite action space, $\gamma\in(0,1)$ is a discount factor, and $F$ is some general utility function. For each state $s\in \cS$, a transition to state $s'\in \cS$ occurs when selecting an action $a\in\cA$ following the distribution $\cP(\cdot | a, s )$. For each state $s\in\cS$, a policy $\pi$ gives a distribution $\pi(\cdot|s)$ over the action space $\cA$. Let $\xi$ be the initial state distribution
and let the unnormalized state-action occupancy measure $\lambda^{\pi}$ be defined by \eqref{defn:occupancy}, we define the general utility function $F$ as a smooth function of the occupancy measure, and the goal of the general utility MDP is to maximize $F(\lambda^\pi)$. With the policy $\pi_\theta$ being parameterized by \eqref{defn:policy-para}, we propose to solve problem \eqref{prob:main-0}, which is 
$$\max_{\theta} R(\pi_{\theta}) := F\left(\lambda^{\pi_{\theta}}\right).$$
For notational convenience, we often write $\lambda(\theta)$ instead of $\lambda^{\pi_{\theta}}$. Such utility function is very general and includes many important problems in RL. We provide a few examples where $F$ are \emph{concave}.
\begin{example}[Cumulative reward]
	\label{example:cumulative-reward}
	When $F\left(\lambda^{\pi_{\theta}}\right) =\left\langle r, \lambda^{\pi_{\theta}}\right\rangle$, for some $r\in\RR^{|\cS||\cA|}$. Then we recover the standard cumulative sum of rewards: 
	$$R(\pi_{\theta})  =F\left(\lambda^{\pi_{\theta}}\right) =
	\EE\big[\sum_{t=0}^{+\infty}\gamma^t\cdot r(s_t,a_t) \,\big|\, \pi_\theta, s_0\sim\xi \big].$$
\end{example}  
\begin{example}[Maximal entropy exploration]
	Let $\mu^{\pi_{\theta}}(s) = (1-\gamma)\sum_{a}\lambda^{\pi_{\theta}}(s,a)$, $\forall s\in\cS$ be the state occupancy measure, which is the margin  of $\lambda^{\pi_{\theta}}$ over $\cS$. Let $F(\cdot)$ be the entropy function, then we recover the objective for maximal entropy exploration \cite{hazan2019provably}:
	$$R(\pi_{\theta})=F\left(\lambda^{\pi_{\theta}}\right) =  -\sum_{s\in\cS} \mu^{\pi_{\theta}}(s)\log\mu^{\pi_{\theta}}(s).$$
\end{example}
\begin{example}[RL with Set Constraint]
	\label{example:SC-RL}
	Let $\mathbf{z}(s_t,a_t)\in\RR^d$ be a vector feedback received in each step. The cumulative feedback is $$\mathbf{u}(\pi_{\theta}):=\EE\big[\sum_{t=0}^{+\infty}\gamma^t\cdot\mathbf{z}(s_t,a_t)|s_0\sim\xi, \pi_\theta\big] = M\lambda^{\pi_{\theta}}$$
	for some matrix $M\in\RR^{d\times |\cS||\cA|}$. \cite{miryoosefi2019reinforcement} proposed a set-constrained RL problem which aims to find a policy $\pi$ s.t. $u(\pi)\in U$ for some convex set $U$. This problem can be formulated as an instance of \eqref{prob:main-0} by letting $F(\cdot)$ be the negative squared distance:
	$$R(\pi_{\theta}) = F\left(\lambda^{\pi_{\theta}}\right) = -\min_{\mathbf{u}'\in U} \|\mathbf{u}'-{\bf u}(\pi_{\theta})\|^2.$$
\end{example}

\section{The TSIVR-PG Algorithm}
In this section, we propose a Truncated Stochastic Incremental Variance-Reduced Policy Gradient (TSIVR-PG) method, which is inspired by techniques of variance reduction and off-policy estimation. A gradient truncation mechanism to proposed to provably control the importance weights in off-policy sampling. 
\subsection{Off-Policy PG Estimation}
\textbf{Policy Gradient}$~$ First, let us derive the policy gradient of the general utility. Let $V^{\pi_{\theta}}(r)$ be the cumulative reward under policy $\pi_{\theta}$, initial distribution $\xi$ and reward function $r$. By Example \ref{example:cumulative-reward}, $V^{\pi_{\theta}}(r) = \langle\lambda(\theta),r\rangle$, the chain rule and policy gradient theorem \cite{sutton2000policy} indicates that
\begin{eqnarray}  
	\label{defn:PGT}
	\nabla_{\theta} V^{\pi_{\theta}}(r)  = \big[\nabla_{\theta}\lambda(\theta)\big]^{\top} r =
	\EE_{\xi,\pi_\theta}\Big[\sum_{t=0}^{+\infty}\gamma^t\cdot r(s_t,a_t)\cdot\Big(\sum_{t'=0}^t\nabla_{\theta}\log \pi_{\theta}(a_{t'}|s_{t'})\Big)\Big],
\end{eqnarray}
where $\nabla_\theta  \lambda(\theta)$ is the Jacobian matrix of the vector mapping $\lambda(\theta)$. That is, policy gradient theorem actually provides a way for computing the Jacobian-vector product for the occupancy measure.  Following the above observation and the chain rule, we have
$$\nabla_{\theta} R(\pi_{\theta}) \!=\! \left[\nabla_\theta  \lambda(\theta)\right]^\top \!\nabla_\lambda F\!\left(\lambda(\theta)\!\right) = \nabla_{\theta} V^{\pi_{\theta}}(r) |_{r = \nabla_\lambda F\!\left(\lambda(\theta)\!\right)}.$$
Therefore, we can estimate the policy gradient using the typical REINFORCE as long as we pick the ``quasi-reward function" as 
$r := \nabla_\lambda F(\lambda(\theta))$.  To find this quasi-reward, we need to estimate the state-action occupancy measure $\lambda(\theta)$ (unless $F$ is linear).

\textbf{Importance Sampling Weight} $~$
Let $\tau = \{s_0, a_0, s_1, a_1, \cdots, s_{H-1}, a_{H-1}\}$ be a  length-$H$ trajectory generated under the initial distribution $\xi$ and the behavioral policy $\pi_{\theta_1}$. 
For any target policy $\pi_{\theta_2}$, we define the importance sampling weight as 
\begin{eqnarray}
	\label{defn:ratio}
	\omega_t(\tau|\theta_1,\theta_2) =  \frac{\Pi_{h=0}^t\pi_{\theta_2}(a_h|s_h)}{\Pi_{t=0}^h\pi_{\theta_1}(a_h|s_h)}, \qquad 0\leq t\leq H-1.
\end{eqnarray}
It is worth noting that such importance sampling weight is inevitable in the stochastic variance reduced policy gradient methods, see \cite{papini2018stochastic,xu2019sample,xu2020improved,pham2020hybrid,liu2019neural}. In these works, the authors usually directly assume $\mathrm{Var}(\omega_{H-1}(\tau|\theta_1,\theta_2))\leq W$ for all the policy pairs encountered in every iteration of their algorithms. However, such assumption is too strong and is uncheckable. 

Based on the above notation of behavioral and target policies, as long as the importance sampling weights, we present the following off-policy occupancy and policy gradient estimators.

\textbf{Off-Policy Occupancy Measure Estimator} $~$
Denote ${\bf e}_{sa}$ the vector with $(s,a)$-th entry being 1 while other entries being 0. We define the following estimator for $\lambda(\theta_2)$ 
\begin{eqnarray}
	\label{defn:weight-lambda-estimator}
	\widehat \lambda_\omega(\tau|\theta_1,\theta_2):= \sum_{t=0}^{H-1}\gamma^t\cdot\omega_t(\tau|\theta_1,\theta_2)\cdot\be_{s_ta_t}.
\end{eqnarray} 
When $\theta_2 = \theta_1$, $\omega_t(\tau|\theta_1,\theta_2)\equiv1$ and $\widehat \lambda_\omega(\tau|\theta_1,\theta_2)$ becomes the on-policy (discounted) empirical distribution, for which we use the simplified notion $\widehat\lambda(\tau|\theta_2):=\widehat \lambda_\omega(\tau|\theta_2,\theta_2)$.

\textbf{Off-Policy Policy Gradient Estimator} $~$
Let $r\in\RR^{|\cS||\cA|}$ be any quasi-reward vector. We aim to estimate the Jacobian-vector product $[\nabla_\theta\lambda(\theta_2)]^\top r$ for target policy $\pi_{\theta_2}$ by
\begin{equation}
	\label{defn:weight-Jacob-estimator}
	\widehat{g}_\omega(\tau|\theta_1,\theta_2,r )  := \sum_{t=0}^{H-1}\gamma^t\cdot\omega_t(\tau|\theta_1,\theta_2) \cdot r(s_t, a_t)\cdot\Big(\sum_{t'=0}^t\nabla_{\theta}\log \pi_{\theta_2}(a_{t'}|s_{t'})\Big).
\end{equation}
When $\theta_2 = \theta_1$, $\omega_t(\tau|\theta_1,\theta_2)\equiv1$ and $\widehat{g}_\omega(\tau|\theta_1,\theta_2,r )$ becomes the on-policy REINFORCE estimator with quasi-reward function $r$. In this case, we use the simplified notion $\widehat g(\tau|\theta_2,r):= \widehat{g}_\omega(\tau|\theta_2,\theta_2,r)$.

Estimators $\widehat \lambda_\omega(\tau|\theta_1,\theta_2)$ and $\widehat{g}_\omega(\tau|\theta_1,\theta_2,r )$ are almost unbiased. In details, 
$$\|\EE_{\tau\sim\pi_{\theta_1}}[\widehat \lambda_\omega(\tau|\theta_1,\theta_2)]\!-\!\lambda(\theta_2)\|\!\leq\!\cO(\gamma^H),\quad  \|\EE_{\tau\sim\pi_{\theta_1}}[\widehat{g}_\omega(\tau|\theta_1,\theta_2,r )]\!-\!\big[\nabla_{\theta} \lambda(\theta_2)\big]^{\!\top}\!r\|\!\leq\!\cO(H\!\cdot\!\gamma^H),$$ see details in Appendix \ref{appdx:prop-weighted-sampling}. Therefore the bias due to truncation is almost negligible if $H$ is properly selected.

\subsection{The TSIVR-PG Algorithm}
To achieve the $\tilde{\cO}(\epsilon^{-2})$ sample complexity, we propose an epoch-wise algorithm called Truncated Stochastic Incremental Variance-Reduced PG (TSIVR-PG) Algorithm. Let $\theta_0^i$ be the starting point of the $i$-th epoch, TSIVR-PG constructs the estimators for $\lambda(\theta_0^i)$, quasi-reward $\nabla_\lambda F(\lambda(\theta_0^i))$ and the policy gradient  $\nabla_{\theta}F(\lambda(\theta_0^i))$ by
\begin{equation}
	\label{defn:SIVR-PG-large-1}
	\lambda_{0}^i = \frac{1}{N}\sum_{\tau\in\cN_i}\widehat\lambda(\tau|\theta_0^i),\,\,\,\,r_0^i = \nabla_\lambda F(\lambda_0^i)\quad\mbox{and}\quad 						g_0^i = \frac{1}{N}\sum_{\tau\in\cN_i} \widehat{g}(\tau|\theta_{0}^i,r_0^i).
\end{equation}
where $\cN_i$ is a set of $N$ independent length-$H$ trajectories sampled under $\pi_{\theta_0^i}$. When $j\geq1$, 
\begin{equation}
	\label{defn:SIVR-PG-small-1}
	\lambda_{j}^i = \frac{1}{B}\sum_{\tau\in\cB_j^i}\left(\widehat\lambda(\tau|\theta^i_j) - \widehat\lambda_\omega(\tau|\theta_{j}^i,\theta_{j-1}^i)\right) + \lambda_{j-1}^i,\qquad r_j^i = \nabla_\lambda F(\lambda_j^i)\vspace{-0.2cm}
\end{equation}
\begin{equation}
	\label{defn:SIVR-PG-small-2}
	g_{j}^i = \frac{1}{B}\sum_{\tau\in\cB_j^i}\Big( \widehat{g}\left(\tau|\theta_{j}^i,r_{j-1}^i\right) -  \widehat{g}_\omega\left(\tau|\theta_j^i,\theta_{j-1}^i,r_{j-2}^i\right)\Big)+g_{j-1}^i,
\end{equation}
where $\cB_j^i$ is a set of $B$ independent length-$H$ trajectories sampled under $\pi_{\theta_j^i}$, and we default $r_{-1}^i:=r_0^i$. Specifically, $\widehat{g}_\omega(\tau|\theta_j^i,\theta_{j-1}^i,r_{j-2}^i)$ is used instead of $\widehat{g}_\omega(\tau|\theta_j^i,\theta_{j-1}^i,r_{j-1}^i)$ for independence issue. The details of the TSIVR-PG algorithm are stated in Algorithm \ref{alg:TSIVR-PG}.
\begin{algorithm2e}
	\caption{The TSIVR-PG Algorithm}
	\label{alg:TSIVR-PG}
	\textbf{Input:} Initial point $\theta_0^1 = \tilde\theta_0$; batch sizes $N$ and $B$; sample trajectory length $H$; stepsize $\eta$;   epoch length $m$; gradient truncation radius $\delta$. \\
	\For{Epoch $i=1,2,...,$}{ 
		\For{Iteration $j=0,...,m-1$}{
			\uIf{$j==0$}{
				Sample $N$ trajectories under policy $\pi_{\theta_{0}^i}$ of length $H$, collected as  $\cN_i$.	\\
				Compute estimators $\lambda_{0}^i$, $r_0^i$ and $g_0^i$ by \eqref{defn:SIVR-PG-large-1}.
				Default $r_{-1}^i := r_0^i$.\\
			}
			\Else{Sample $B$ trajectories under policy $\pi_{\theta_{j}^i}$ with length $H$, collected as $\cB_{j}^i$.\\
				Compute estimators $\lambda_{j}^i$, $r_j^i$ and $g_j^i$ by \eqref{defn:SIVR-PG-small-1} and \eqref{defn:SIVR-PG-small-2}. 
			}
			Update the policy parameter by a truncated gradient ascent step:\vspace{-0.2cm}
			{\small \begin{equation}
					\label{defn:GA-Trc}
					\theta_{j+1}^i = \begin{cases}
						\theta_j^i + \eta\cdot g_j^i&, \mbox{ if } \eta\|g_j^i\|\leq\delta,\\
						\theta_j^i + \delta\cdot g_j^i/\|g_j^i\|&,\mbox{ otherwise.}
					\end{cases}\vspace{-0.4cm}
			\end{equation}} 
		}
		Set $\theta_0^{i+1} = \tilde \theta_{i} = \theta_{m}^i$. \\
	}
\end{algorithm2e}

It is worth noting that the truncated gradient step \eqref{defn:GA-Trc} is equivalent to a trust region subproblem:
\begin{equation}
	\label{defn:TR}
	\theta_{j+1}^i = \argmax_{\|\theta-\theta_j^i\|\leq\delta} F(\lambda(\theta_j^i)) + \langle g_j^i, \theta-\theta_j^i\rangle + \frac{1}{2\eta}\|\theta-\theta_j^i\|^2
\end{equation}
where the approximate Hessian matrix is simply chosen as $(\eta)^{-1}\cdot I$.

\section{Sample Efficiency of TSIVR-PG}
\label{sec:complexity}
In this section, we analyze the finite-sample performance of TSIVR-PG. We first show that TSIVR-PG finds an $\epsilon$-stationary solution with $\tilde\cO(\epsilon^{-3})$ samples. Given additional assumptions, we show that TSIVR-PG finds a global $\epsilon$-optimal solution with $\tilde{\cO}(\epsilon^{-2})$ samples .
\subsection{Convergence Towards Stationary Points}
Since we focus on the soft-max policy parameterization where
${\small\pi_{\theta}(a|s) = \frac{\exp\{\psi(s,a;\theta)\}}{\sum_{a'}\exp\{\psi(s,a';\theta)\}}}$, we make the following assumptions on the parameterization function $\psi$ and the utility $F$.
\begin{assumption}
	\label{assumption:psi-Lip}
	$\psi(s,a;\cdot)$ is twice differentiable for all  $s$ and $a$. There $\exists \,\ell_\psi,L_\psi>0$ s.t. 
	\begin{equation}
		\label{defn:psi-Lip}
		\max_{s\in\cS,a\in\cA}\sup_{\theta}\|\nabla_{\theta}\psi(s,a;\theta)\|\leq \ell_\psi\quad\mbox{and}\quad \max_{s\in\cS,a\in\cA}\sup_{\theta}\|\nabla_{\theta}^2\psi(s,a;\theta)\|\leq L_h,
	\end{equation}
	where $\|\cdot\|$ stands for $L_2$ norm and spectral norm for vector and matrix respectively.
\end{assumption}
\begin{assumption}
	\label{assumption:F-Lip}
	$F$ is a smooth and possibly nonconvex function. There exists $\ell_{\lambda,\infty}>0$ such that
	$\|\nabla_\lambda F(\lambda)\|_\infty\leq \ell_{\lambda,\infty}$. And there exist constants $L_{\lambda,\infty},L_{\lambda}>0$ s.t. it holds that $$\begin{cases}\|\nabla_\lambda F(\lambda) - \nabla_\lambda F(\lambda')\|_\infty\leq L_{\lambda}\|\lambda-\lambda'\|_2\\
	\|\nabla_\lambda F(\lambda) - \nabla_\lambda F(\lambda')\|_\infty\leq L_{\lambda,\infty}\|\lambda-\lambda'\|_1\end{cases}\mbox{ for }\quad\forall \lambda,\lambda'.$$
\end{assumption}
As a consequence, we have the following lemmas. 
\begin{lemma}
	\label{lemma:importance}
	Given Assumption \ref{assumption:psi-Lip} and \ref{assumption:F-Lip}, the following results hold:\\
	{\bf(i).} For any policy parameter $\theta$ and any state-action pair $(s,a)$, then it holds for any $s,a$ and $\theta$ that
	$$\begin{cases}\|\nabla_\theta \log\pi_\theta(a|s)\|\leq 2\ell_\psi \\
	\|\nabla_\theta^2 \log\pi_\theta(a|s)\|\leq 2(L_\psi + \ell_\psi^2)\end{cases}\quad\mbox{and}\qquad\|\nabla_\theta F(\lambda(\theta))\|\leq\frac{2\ell_\psi\cdot \ell_{\lambda,\infty}}{(1-\gamma)^2}.$$\\
	{\bf(ii).} For any policy parameters $\theta_1$ and $\theta_2$, it holds that  $$\|\lambda^{\pi_{\theta_1}}-\lambda^{\pi_{\theta_2}}\|_1\leq  \frac{2\ell_\psi}{(1-\gamma)^2}\cdot\|\theta_1-\theta_2\|.$$ \\ 
	{\bf (iii).} The objective function $F\circ\lambda(\cdot)$ is $L_\theta$-smooth, with 
	$$L_\theta = \frac{4L_{\lambda,\infty}\cdot\ell_\psi^2}{(1-\gamma)^4} + \frac{8\ell_{\psi}^2\cdot \ell_{\lambda,\infty}}{(1-\gamma)^3}+\frac{2\ell_{\lambda,\infty}\cdot(L_\psi+\ell_{\psi}^2)}{(1-\gamma)^2}.$$
\end{lemma}
\noindent To measure the convergence, we propose to use the gradient mapping defined as follows:
$$\cG_\eta(\theta) = \frac{\theta_+ - \theta}{\eta}, \quad\mbox{ where }\quad \theta_+ =\theta+ \begin{cases}
	\theta + \eta\cdot g&, \mbox{ if } \eta\|g\|\leq\delta,\\
	\theta + \delta\cdot g/\|g\|&,\mbox{ otherwise}\end{cases}$$
where $g = \nabla_\theta F(\lambda(\theta))$. We remark that, $\EE[\|\cG_\eta(\theta_j^i)\|^2]$ is more suitable for the ascent analysis of the truncated gradient updates, compared with the commonly used $\EE[\|\nabla_\theta F(\lambda(\theta_j^i))\|^2]$. Note that $\cG_\eta(\theta) = \nabla F(\lambda(\theta))$ if $\|\cG_\eta(\theta)\|\leq\delta$ and $\|\nabla F(\lambda(\theta))\|$ is bounded for any $\theta$. Based on such observation, we have the following lemma to validate the choice of the proposed stationarity measure.
\begin{lemma}
	\label{lemma:optimality-measure} 
	For any random vector $\theta$, if $\EE[\|\cG_\eta(\theta)\|]\leq\epsilon$, then $$\EE[\|\nabla_\theta F(\lambda(\theta))\|]\leq \cO(\delta^{-1}\cdot\epsilon).$$
\end{lemma}
Based on the notion of $\cG_\eta$, we characterize the per-iteration ascent as follows. 
\begin{lemma}
	\label{lemma:ascent-ncvx}
	Let the iterates be generated by Algorithm \ref{alg:TSIVR-PG}. Then it holds that 
	\begin{eqnarray*}
		F(\lambda(\theta_{j+1}^i))&\geq& F(\lambda(\theta_j^i)) + \frac{\eta}{4}\|\cG_\eta(\theta_j^i)\|^2 + \Big(\frac{1}{2\eta}-L_\theta\Big)\|\theta_{j+1}^i-\theta_j^i\|^2 \\
		&&- \Big(\frac{\eta}{2} + \frac{1}{2L_\theta}\Big)\|\nabla_\theta F(\lambda(\theta_j^i))-g_j^i\|^2.
	\end{eqnarray*}
\end{lemma}
This suggests us to bound mean-squared-error $\EE[\|\nabla_\theta F(\lambda(\theta_j^i))\!-\!g_j^i\|^2]$. For this purpose, we need to bound the importance sampling weight, by utilizing the soft-max form of policy parameterization \eqref{defn:policy-para}. 
\begin{lemma}
	\label{lemma:Weight-Var0}
	For any behavioral policy $\pi_{\theta_1}$ and target policy $\pi_{\theta_2}$ parameterized by \eqref{defn:policy-para}, the importance weight satisfies $$\omega_t(\tau|\theta_1,\theta_2)\leq\exp\big\{2(t+1)\ell_\psi\|\theta_1-\theta_2\|\big\},$$ for $\forall 0\leq t\leq H-1$.
\end{lemma}
Since TSIVR-PG only uses importance weights for two consecutive iterations $\theta^i_j, \theta^i_{j-1}$ while forcing $\|\theta^i_j \!-\! \theta^i_{j-1}\|\!\leq\!\delta$ by the truncated gradient step \eqref{defn:GA-Trc}, we have $\omega_{\!H\!-\!1}(\tau|\theta^i_j,\!\theta^i_{j-1})\!\leq\! \exp\{2H\ell_\psi\delta\}$ w.p. 1. As we will see later, the effective horizon $H$ only has a mild magnitude of $\cO\big((1-\gamma)^{-1}\cdot\log(1/\epsilon)\big)$, the truncation radius only need to satisfy $\delta = \cO(H^{-1}\ell_\psi^{-1})$ s.t. $\omega_{t\!-\!1}(\tau|\theta^i_j,\!\theta^i_{j-1})\! = \!\cO(1)$, for $\forall t\!\leq\! H\!-\!1.$ 
Consequently, combining Lemma \ref{lemma:importance}, \ref{lemma:Weight-Var0} and Lemma B.1 of \cite{xu2019sample} gives the following result.
\begin{lemma}
	\label{lemma:Weight-Var}
	Let policy $\pi_\theta$ be parameterized by \eqref{defn:policy-para} with function $\psi$ satisfying Assumption \ref{assumption:psi-Lip}. Suppose behavioral policy $\pi_{\theta_1}$ and target policy $\pi_{\theta_2}$ satisfy $\|\theta_1-\theta_2\|\leq\delta$, then
	$$\EE[\omega_t(\tau|\theta_1,\theta_2)] = 1\quad\mbox{and}\quad\mathrm{Var}\left(\omega_t(\tau|\theta_1,\theta_2)\right) \leq C_\omega(t+1)\cdot\|\theta_1-\theta_2\|^2, $$
	where $\tau$ is sampled under policy $\pi_{\!\theta_1}$, and $C_\omega(t) = t\big(4\ell_\psi^2(t+\frac{1}{2}) +2L_\psi\big)(e^{4\delta t} + 1)$.	
\end{lemma}
As a result, we can bound the mean-squared-error of the $g_j^i$ as follows. 
\begin{lemma}
	\label{lemma:SIVR-variance}
	For the PG estimators $g_{j}^i$, we have 
	\begin{align} 
		\EE\Big[\|g_j^i-\nabla_{\theta} F(\lambda(\theta_{j}^i))\|^2\Big] \,\,\leq\,\, \frac{C_1}{N} &+ C_2\gamma^{2H}+ \frac{C_3}{B}\cdot\sum_{j'=1}^j\EE\left[\|\theta_{j'-1}^i-\theta_{j'}^i\|^2\right] \nonumber\\
		&+ C_4\cdot\EE\left[\|\theta_{j-1}^i-\theta_{j}^i\|^2\right]\nonumber
	\end{align}
	for some constants  
	$C_1,..,C_4>0$. In case $j = 0$, we default $\sum^{0}_{j'=1}\cdot = 0$.
\end{lemma}
The expression of constants $C_i$'s are complicated, we provide their detailed formula in the appendix. If we set $H = \cO\big(\frac{\log(1/\epsilon)}{1-\gamma}\big)$ and $\delta \leq \frac{1}{2H\ell_\psi}$, then $C_i$ only depends polynomially on the Lipschitz constants, $\log(\epsilon^{-1})$, and $(1-\gamma)^{-1}$. Combining Lemma \ref{lemma:ascent-ncvx}, \ref{lemma:SIVR-variance}, and \ref{lemma:optimality-measure} gives the following theorem. 
\begin{theorem}
	\label{theorem:ncvx}
	For Algorithm \ref{alg:TSIVR-PG}, we choose $H = \frac{2\log(1/\epsilon)}{1-\gamma}$, $\delta = \frac{1}{2H\ell_\psi}$, $B = m = \epsilon^{-1}$, $N = \epsilon^{-2}$, $\eta = \frac{1}{1+(C_3+C_4)/L_\theta^{2}}\cdot\frac{1}{2L_\theta}$. After running the algorithm for $T=\epsilon^{-1}$ epochs and output $\theta_{out}$ from $\{\theta_j^i\}_{j=0,\cdots,m-1}^{i=1,\cdots,T}$ uniformly at random, then $\EE[\|\cG_\eta(\theta_{out})\|]\leq\cO(\epsilon)$. The total number of samples used is $Tm\!\cdot\!(BH \!+\! N) = \tilde{\cO}(\epsilon^{-3}).$
	By Lemma \ref{lemma:optimality-measure}, we also have $\EE[\|\nabla_\theta F(\lambda(\theta_{out}))\|]\leq\cO(\epsilon)$. 
\end{theorem}

\subsection{Convergence Towards Global Optimality}
Next, we provide a mechanism to establish the convergence of TSIVR-PG to global optimality. For this purpose, we introduce the hidden convexity of the general utility RL problem. In addition to the smoothness of $F$ (Assumption \ref{assumption:F-Lip}), we further assume its concavity, formally stated as follows. 
\begin{assumption}
	\label{assumption:convexity}
	$F$ is a concave function of the state-action occupancy measure.
\end{assumption}
Let $\cL$ be the image of the mapping $\lambda(\theta)$. Then the parameterized \emph{policy optimization} problem \eqref{prob:main-0} can be rewritten as an equivalent \emph{occupancy optimization} problem:
\begin{equation}
	\label{prob:main-1}
	\mathrm{max}_\theta\,\, F(\lambda(\theta))\qquad\Longleftrightarrow\qquad\mathrm{max}_{\mu\in\cL} \,\,F(\mu).
\end{equation}
When the policy parameterization is powerful enough to represent any policy, the image $\cL$ is a convex polytope, see e.g. \cite{chen2018scalable}. Since $F$ is concave, the occupancy optimization problem is a \emph{convex optimization} problem. In this case, if the mapping $\lambda(\cdot)$ is invertible (see \cite{zhang2020variational}), we may view the original problem \eqref{prob:main-0} as a reformulation of a convex problem by a change of variable: $\theta = \lambda^{-1}(\mu)$. We call this property ``hidden convexity''.  However, requiring $\lambda(\cdot)$ to be invertible is too restrictive, and it doesn't even hold for simple soft-max policy with $\psi(s,a;\theta) = \theta_{sa}$ where multiple $\theta$ correspond to a same policy. Therefore, we adopt a weaker assumption where (i). $\pi_{\theta}$ can represent any policy (ii). a continuous inverse  $\lambda^{-1}(\cdot)$ can be locally defined over a subset of $\theta$.
\begin{assumption}
	\label{assumption:over-para}
	For policy parameterization of form \eqref{defn:policy-para}, $\theta$ overparametrizes the set of policies in the following sense. (i). For any $\theta$ and $\lambda(\theta)$, there exist (relative) neighourhoods  $\theta\in\mathcal{U}_\theta\subset B(\theta,\delta)$ and $\lambda(\theta)\in\mathcal{V}_{\lambda(\theta)}\subset\lambda(B(\theta,\delta))$ s.t.  $\big(\lambda|_{\mathcal{U}_\theta}\big)(\cdot)$ forms a bijection between  $\mathcal{U}_\theta$ and $\mathcal{V}_{\lambda(\theta)}$, where $\big(\lambda|_{\mathcal{U}_\theta}\big)(\cdot)$ is the confinement of $\lambda$ onto $\mathcal{U}_\theta$. We assume $(\lambda|_{\mathcal{U}_\theta})^{-1}(\cdot)$ is $\ell_\theta$-Lipschitz continuous for any $\theta$.
	(ii). Let $\pi_{\theta^*}$ be the optimal policy. Assume there exists $\bar\epsilon$ small enough, s.t. $(1-\epsilon)\lambda(\theta) + \epsilon\lambda(\theta^*)\in\mathcal{V}_{\lambda(\theta)}$ for $\forall \epsilon\leq\bar\epsilon$, $\forall\theta$. 
\end{assumption}
Based on Assumption \ref{assumption:over-para}, we replace Lemma \ref{lemma:ascent-ncvx} with the following lemma.
\begin{lemma}
	\label{lemma:ascent-0}
	For $\forall\epsilon<\bar\epsilon$ with $\bar\epsilon$ defined in Assumption \ref{assumption:over-para}, it holds for all iterations that
	\begin{align} 
		\label{lm:ascent}
		&F(\lambda(\theta^*))-F(\lambda(\theta^i_{j+1})) \,\,\leq\,\,  (1-\epsilon)\left(F(\lambda(\theta^*))-F(\lambda(\theta_{j}^i))\right)  \\ 
		& \qquad\qquad\qquad+ \Big(L_\theta+ \frac{1}{2\eta}\Big)\frac{2\epsilon^2\ell_{\theta}^2}{(1-\gamma)^2} - \Big(\frac{1}{2\eta}- L_\theta\Big)\|\theta_{j+1}^i - \theta_j^i\|^2 + \frac{1}{L_\theta}\|g_j^i - \nabla_\theta F(\lambda(\theta_j^i))\|^2.\nonumber
	\end{align}
\end{lemma}
The analysis of Lemma \ref{lemma:ascent-0} is very different from its nonconvex optimization counterpart (Lemma \ref{lemma:ascent-ncvx}).  Next, we derive the sample complexity of the TSIVR-PG algorithm given Lemma \ref{lemma:ascent-0} and \ref{lemma:SIVR-variance}. 

\begin{theorem}
	\label{theorem:convergence-rate}
	For TSIVR-PG method (Algorithm \ref{alg:TSIVR-PG}), let $\epsilon\in(0,\bar\epsilon)$ be the target accuracy. If we choose $H,m,B,N$ and $\delta$ according to Theorem \ref{theorem:ncvx}. and we let the stepsize to be small enough s.t. $\eta \leq \frac{1}{2L_\theta + 8(C_3+C_4)/L_\theta}$, then after at most $ T = \log_2(\epsilon^{-1})$ epochs,  $\EE\big[F(\lambda(\theta^*))\!-\!F(\lambda(\tilde\theta_{T}))\big]\leq\cO(\epsilon)$.
	The total number of samples taken is $T\times ((m-1)B+N)\times H = \tilde{\cO}(\epsilon^{-2})$.
\end{theorem}

\section{Numerical Experiments}
\label{sec:experiments}

\subsection{Maximizing Cumulative Reward.} 
In this experiment, we aim to evaluate the performance of the TSIVR-PG algorithm for maximizing the cumulative sum of reward. As the benchmarks, we also implement the SVRPG \cite{xu2020improved}, the SRVR-PG \cite{xu2019sample}, the HSPGA \cite{pham2020hybrid}, and the REINFORCE \cite{williams1992simple} algorithms.  Our experiment is performed on benchmark RL environments including the FrozenLake,  Acrobot and Cartpole that are available from OpenAI gym \cite{brockman2016openai}, which is a well-known toolkit for developing and comparing reinforcement learning algorithms.  For all the algorithms, their batch sizes are chosen according to their theory.  In details, let $\epsilon$ be any target accuracy.  For both TSIVR-PG and SRVR-PG, we set $N = \Theta(\epsilon^{-2})$, $B = m =  \Theta(\epsilon^{-1})$. For SVRPG, we set $N=\Theta(\epsilon^{-2})$, $B = \Theta(\epsilon^{-4/3})$ and $m = \Theta(\epsilon^{-2/3})$. For HSPGA, we set $B =  \Theta(\epsilon^{-1})$, other parameters are calculated  according to formulas in \cite{pham2020hybrid} given $B$. For REINFORCE, we set the batchsize to be $N = \Theta(\epsilon^{-2})$.  The parameter $\varepsilon$ and the stepsize/learning rate are tuned for each individual algorithm using a grid search.  For each algorithm, we run the experiment for multiple times with random initialization of the policy parameters.  The curve is obtained by first calculating the moving average of the most recent 50 episodes, and then calculate the median of the return over the outcomes of different runs.  The upper and lower bounds of the shaded area are calculated as the $\frac{1}{4}$ and $\frac{3}{4}$ quantiles over the outcomes. We run the experiment for 10 times for the FrozenLake environment and 50 times for the other environments.  The detailed parameters used in the experiments are presented in the Appendix.  

\paragraph{FrozenLake} $~\,\,$The FrozenLake8x8 environment is a tabular MDP with finite state and action spaces.  For this environment, the policy is parameterized with $\psi(s,a;\theta) = \theta_{sa}$.  

\paragraph{Cartpole and Acrobot} $~\,\,$Both the Cartpole environment and the Acrobot environment are environments with a discrete action space and a continuous state space.  For both environments,  we use a neural network with two hidden layers with width 64 for both layers to model the policy. 

\paragraph{Result} $~\,\,$We plot our experiment outcomes in Figure \ref{exp_fig_1}. The experiments show that given enough episodes, all of the algorithms are able to solve the tasks, achieving nearly optimal returns.  And as expected,  the REINFORCE algorithm takes the longest time to find the optimal policy.  While the other algorithms yield a faster convergence speed, the TSIVR-PG algorithm consistently outperforms the other benchmark algorithms under all of the environments, showing the advantage of our method.  
\begin{figure}[htb!]
	\label{exp_fig_1}
	\center
	\hspace{-0.5cm}
	\includegraphics[scale=0.35]{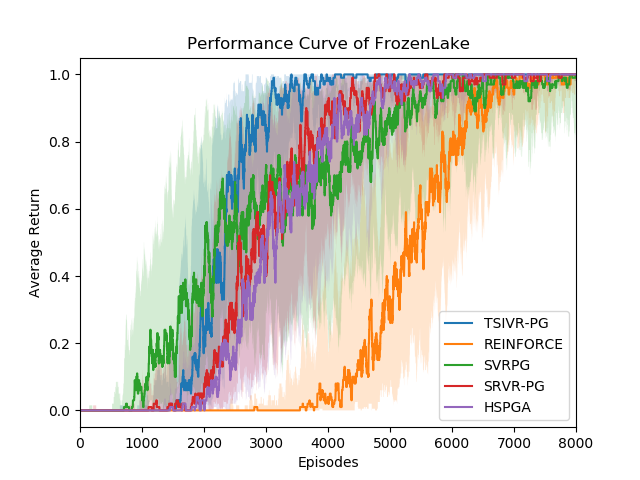}\hspace{-0.5cm}
	\includegraphics[scale=0.35]{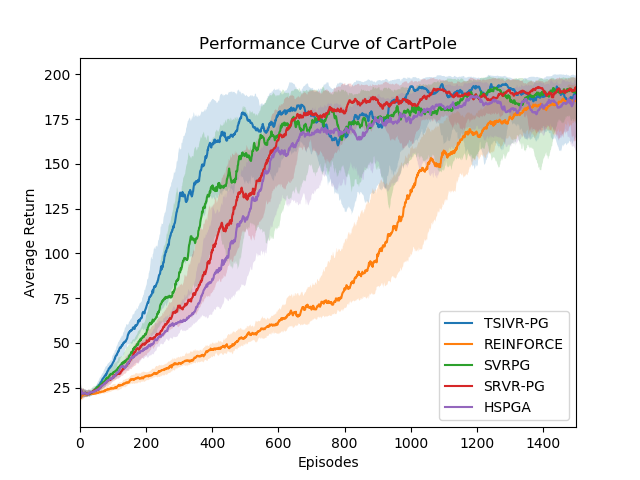}\hspace{-0.5cm}
	\includegraphics[scale=0.35]{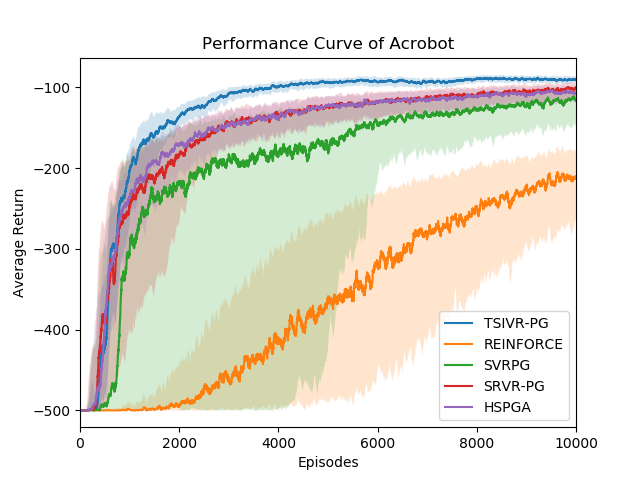}\hspace{-0.5cm}
	\caption{The performance curves of TSIVR-PG and benchmark algorithms under different environments.  The curve is the median return over multiple runs and the shaded areas are calculated as the $\frac{1}{4}$ and $\frac{3}{4}$ quantiles of the experiment outcomes. }
\end{figure}

\subsection{Validating the $\tilde{\cO}(\epsilon^{-2})$ Sample Complexity}
Besides the comparison between different benchmark algorithms, we also perform a validation experiment showing that for certain environments, the convergence rate of TSIVR-PG is close to the theoretical guarantee.  Because the parameters $N, B, m$ are dependent on the target accuracy $\epsilon$, in this section we adopt a different way to set up these parameters: We first set a fixed epoch $E$,  and perform experiments using different values of the parameter $N$.  The parameter $B$ and $m$ are set according to our choice of $N$ by $B=m=\sqrt{N}$.  The performance of the algorithm output is calculated as the average score of the last few episodes, which is then averaged over 10 independent runs.  Again, we use the FrozenLake8x8 environment to do the experiment.  Because FrozenLake8x8 is a tabular environment whose transition and reward function can be easily obtained from the document, we can calculate it's optimal value simply by value iteration, which takes $0.4146$ when we choose $\gamma=0.99$.  In this way, we calculate the gap between the algorithm return and the optimal value, and get log-log figure w.r.t. the gap and the number of episodes calculated by $E(N+Bm)=2EN$. 

\paragraph{Result} $~\,\,$The result is shown in the first sub-figure of Figure \ref{exp_gif_2}, where the blue curve is the gap between the average return of experiment outcome and the optimal value and the shaded area is the range of one standard deviation of the logarithm value.  In addition, we add a orange dotted line to fit the convergence curve, whose slope takes value $-0.496$,  which nearly matches the $O(\epsilon^{-2})$ theoretical bound (slope $-0.5$).

\subsection{Maximizing Non-linear Objective Function }
The TSIVR-PG algorithm is designed not only to solve typical RL problems, but is also able to solve a broader class of problems where the objective function is a general concave function. Unfortunately, none of the benchmark algorithms proposed in the previous section have the ability to solve this kind of problem. To evaluate the performance of our algorithm, we choose another benchmark algorithm, which is the MaxEnt algorithm \cite{hazan2019provably}.  In the experiment, we use FrozenLake8x8 environment since it's more tractable to compute $\lambda$ for a discrete state space. We set the objective function as
\begin{align*}
	F(\lambda) = \sum_{s\in\mathcal{S}}\log \bigg(\sum_{a\in\mathcal{A}}\lambda_{s, a} + \sigma\bigg),
\end{align*}
where $\sigma$ is a fixed small constant.  We choose $\sigma=0.125$ in our experiment. The orders of $N, B, m$ are set in the same way as those in section 6.1.  For the MaxEnt algorithm,  note that in the original paper, the nonlinear objective function assumes the input value is the stationary state distribution $d^\pi$,  but the input value can easily be changed into our $\lambda$ without changing the steps of the algorithm much.  The result is illustrated in Fig. \ref{exp_gif_2}.  From the result, we may see that our algorithm consistently outperforms the benchmark.

\begin{figure}[htb!]
	\label{exp_gif_2}
	\center
	\includegraphics[scale=0.45]{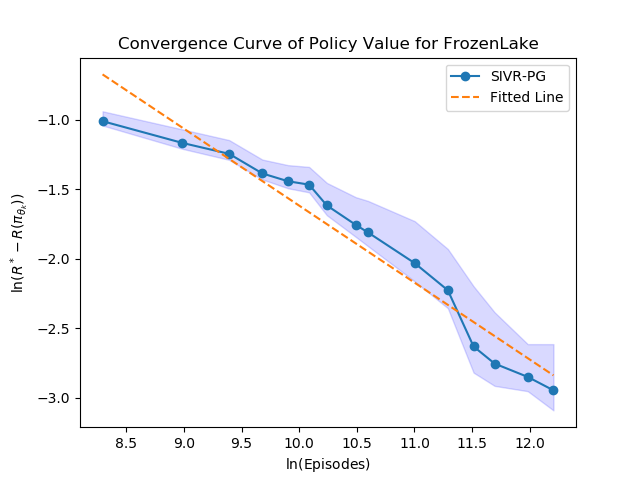}
	\includegraphics[scale=0.45]{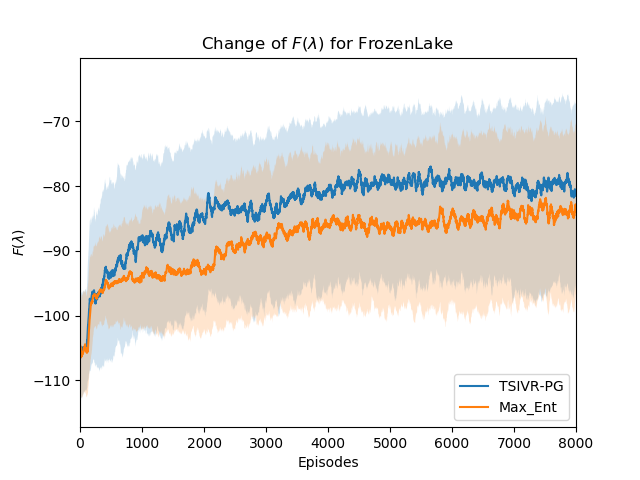}
	\caption{Left: Empirical Evaluation of the Convergence Rate of TSIVR-PG. The optimality gap achieved by TSIVR-PG decreases as the sample size increases,  nearly matching the $\epsilon^{-2}$ sample complexity theory (orange line).  Right: Performance Curve ofTSIVR-PG and MaxEnt for Maximizing Non-linear Objective Functions.  The curve is the median return over 10 runs and the shaded areas are calculated as the $\frac{1}{4}$ and $\frac{3}{4}$ quantiles of the experiment outcomes. }
\end{figure}

\bibliographystyle{plain}
\bibliography{bibliography} 

\begin{thebibliography}{10}

\bibitem{agarwal2020optimality}
Alekh Agarwal, Sham~M Kakade, Jason~D Lee, and Gaurav Mahajan.
\newblock Optimality and approximation with policy gradient methods in markov
  decision processes.
\newblock In {\em Conference on Learning Theory}, pages 64--66. PMLR, 2020.

\bibitem{allen2016variance}
Zeyuan Allen-Zhu and Elad Hazan.
\newblock Variance reduction for faster non-convex optimization.
\newblock In {\em International conference on machine learning}, pages
  699--707, 2016.

\bibitem{azar2011reinforcement}
Mohammad~Gheshlaghi Azar, R{\'e}mi Munos, Mohammad Ghavamzadeh, and Hilbert
  Kappen.
\newblock Reinforcement learning with a near optimal rate of convergence.
\newblock 2011.

\bibitem{azar2013minimax}
Mohammad~Gheshlaghi Azar, R{\'e}mi Munos, and Hilbert~J Kappen.
\newblock Minimax pac bounds on the sample complexity of reinforcement learning
  with a generative model.
\newblock {\em Machine learning}, 91(3):325--349, 2013.

\bibitem{baxter2001infinite}
Jonathan Baxter and Peter~L Bartlett.
\newblock Infinite-horizon policy-gradient estimation.
\newblock {\em Journal of Artificial Intelligence Research}, 15:319--350, 2001.

\bibitem{bhandari2019global}
Jalaj Bhandari and Daniel Russo.
\newblock Global optimality guarantees for policy gradient methods.
\newblock {\em arXiv preprint arXiv:1906.01786}, 2019.

\bibitem{bhandari2020note}
Jalaj Bhandari and Daniel Russo.
\newblock A note on the linear convergence of policy gradient methods.
\newblock {\em arXiv preprint arXiv:2007.11120}, 2020.

\bibitem{brockman2016openai}
Greg Brockman, Vicki Cheung, Ludwig Pettersson, Jonas Schneider, John Schulman,
  Jie Tang, and Wojciech Zaremba.
\newblock Openai gym.
\newblock {\em arXiv preprint arXiv:1606.01540}, 2016.

\bibitem{cen2020fast}
Shicong Cen, Chen Cheng, Yuxin Chen, Yuting Wei, and Yuejie Chi.
\newblock Fast global convergence of natural policy gradient methods with
  entropy regularization.
\newblock {\em arXiv preprint arXiv:2007.06558}, 2020.

\bibitem{chen2018scalable}
Yichen Chen, Lihong Li, and Mengdi Wang.
\newblock Scalable bilinear pi learning using state and action features.
\newblock {\em arXiv preprint arXiv:1804.10328}, 2018.

\bibitem{defazio2014saga}
Aaron Defazio, Francis Bach, and Simon Lacoste-Julien.
\newblock Saga: A fast incremental gradient method with support for
  non-strongly convex composite objectives.
\newblock {\em Advances in neural information processing systems},
  27:1646--1654, 2014.

\bibitem{even2006action}
Eyal Even-Dar, Shie Mannor, Yishay Mansour, and Sridhar Mahadevan.
\newblock Action elimination and stopping conditions for the multi-armed bandit
  and reinforcement learning problems.
\newblock {\em Journal of machine learning research}, 7(6), 2006.

\bibitem{fang2018spider}
Cong Fang, Chris~Junchi Li, Zhouchen Lin, and Tong Zhang.
\newblock Spider: Near-optimal non-convex optimization via stochastic
  path-integrated differential estimator.
\newblock In {\em Advances in Neural Information Processing Systems}, pages
  689--699, 2018.

\bibitem{fazel2018global}
Maryam Fazel, Rong Ge, Sham~M Kakade, and Mehran Mesbahi.
\newblock Global convergence of policy gradient methods for the linear
  quadratic regulator.
\newblock {\em arXiv preprint arXiv:1801.05039}, 2018.

\bibitem{ghadimi2020single}
Saeed Ghadimi, Andrzej Ruszczynski, and Mengdi Wang.
\newblock A single timescale stochastic approximation method for nested
  stochastic optimization.
\newblock {\em SIAM Journal on Optimization}, 30(1):960--979, 2020.

\bibitem{hazan2019provably}
Elad Hazan, Sham Kakade, Karan Singh, and Abby Van~Soest.
\newblock Provably efficient maximum entropy exploration.
\newblock In {\em International Conference on Machine Learning}, pages
  2681--2691. PMLR, 2019.

\bibitem{huo2018accelerated}
Zhouyuan Huo, Bin Gu, Ji~Liu, and Heng Huang.
\newblock Accelerated method for stochastic composition optimization with
  nonsmooth regularization.
\newblock In {\em Proceedings of the AAAI Conference on Artificial
  Intelligence}, volume~32, 2018.

\bibitem{j2016proximal}
Sashank J~Reddi, Suvrit Sra, Barnabas Poczos, and Alexander~J Smola.
\newblock Proximal stochastic methods for nonsmooth nonconvex finite-sum
  optimization.
\newblock {\em Advances in neural information processing systems},
  29:1145--1153, 2016.

\bibitem{johnson2013accelerating}
Rie Johnson and Tong Zhang.
\newblock Accelerating stochastic gradient descent using predictive variance
  reduction.
\newblock {\em Advances in neural information processing systems}, 26:315--323,
  2013.

\bibitem{kakade2001natural}
Sham~M Kakade.
\newblock A natural policy gradient.
\newblock {\em Advances in neural information processing systems},
  14:1531--1538, 2001.

\bibitem{konda2000actor}
Vijay~R Konda and John~N Tsitsiklis.
\newblock Actor-critic algorithms.
\newblock In {\em Advances in neural information processing systems}, pages
  1008--1014, 2000.

\bibitem{konda1999actor}
Vijaymohan~R Konda and Vivek~S Borkar.
\newblock Actor-critic--type learning algorithms for markov decision processes.
\newblock {\em SIAM Journal on control and Optimization}, 38(1):94--123, 1999.

\bibitem{lian2017finite}
Xiangru Lian, Mengdi Wang, and Ji~Liu.
\newblock Finite-sum composition optimization via variance reduced gradient
  descent.
\newblock In {\em Artificial Intelligence and Statistics}, pages 1159--1167.
  PMLR, 2017.

\bibitem{liu2019neural}
Boyi Liu, Qi~Cai, Zhuoran Yang, and Zhaoran Wang.
\newblock Neural trust region/proximal policy optimization attains globally
  optimal policy.
\newblock In {\em Advances in Neural Information Processing Systems}, pages
  10565--10576, 2019.

\bibitem{mei2020global}
Jincheng Mei, Chenjun Xiao, Csaba Szepesvari, and Dale Schuurmans.
\newblock On the global convergence rates of softmax policy gradient methods.
\newblock {\em arXiv preprint arXiv:2005.06392}, 2020.

\bibitem{miryoosefi2019reinforcement}
Sobhan Miryoosefi, Kiant{\'e} Brantley, Hal Daume~III, Miro Dudik, and Robert~E
  Schapire.
\newblock Reinforcement learning with convex constraints.
\newblock In {\em Advances in Neural Information Processing Systems}, pages
  14093--14102, 2019.

\bibitem{mnih2016asynchronous}
Volodymyr Mnih, Adria~Puigdomenech Badia, Mehdi Mirza, Alex Graves, Timothy
  Lillicrap, Tim Harley, David Silver, and Koray Kavukcuoglu.
\newblock Asynchronous methods for deep reinforcement learning.
\newblock In {\em International conference on machine learning}, pages
  1928--1937, 2016.

\bibitem{nguyen2017sarah}
Lam~M Nguyen, Jie Liu, Katya Scheinberg, and Martin Tak{\'a}{\v{c}}.
\newblock Sarah: A novel method for machine learning problems using stochastic
  recursive gradient.
\newblock {\em arXiv preprint arXiv:1703.00102}, 2017.

\bibitem{nguyen2019finite}
Lam~M Nguyen, Marten van Dijk, Dzung~T Phan, Phuong~Ha Nguyen, Tsui-Wei Weng,
  and Jayant~R Kalagnanam.
\newblock Finite-sum smooth optimization with sarah.
\newblock {\em arXiv preprint arXiv:1901.07648}, 2019.

\bibitem{papini2018stochastic}
Matteo Papini, Damiano Binaghi, Giuseppe Canonaco, Matteo Pirotta, and Marcello
  Restelli.
\newblock Stochastic variance-reduced policy gradient.
\newblock {\em arXiv preprint arXiv:1806.05618}, 2018.

\bibitem{peters2008natural}
Jan Peters and Stefan Schaal.
\newblock Natural actor-critic.
\newblock {\em Neurocomputing}, 71(7-9):1180--1190, 2008.

\bibitem{pham2020hybrid}
Nhan~H Pham, Lam~M Nguyen, Dzung~T Phan, Phuong~Ha Nguyen, Marten van Dijk, and
  Quoc Tran-Dinh.
\newblock A hybrid stochastic policy gradient algorithm for reinforcement
  learning.
\newblock {\em arXiv preprint arXiv:2003.00430}, 2020.

\bibitem{reddi2016stochastic}
Sashank~J Reddi, Ahmed Hefny, Suvrit Sra, Barnabas Poczos, and Alex Smola.
\newblock Stochastic variance reduction for nonconvex optimization.
\newblock In {\em International conference on machine learning}, pages
  314--323, 2016.

\bibitem{reddi2016fast}
Sashank~J Reddi, Suvrit Sra, Barnab{\'a}s P{\'o}czos, and Alex Smola.
\newblock Fast incremental method for nonconvex optimization.
\newblock {\em arXiv preprint arXiv:1603.06159}, 2016.

\bibitem{schmidt2017minimizing}
Mark Schmidt, Nicolas Le~Roux, and Francis Bach.
\newblock Minimizing finite sums with the stochastic average gradient.
\newblock {\em Mathematical Programming}, 162(1-2):83--112, 2017.

\bibitem{schulman2015trust}
John Schulman, Sergey Levine, Pieter Abbeel, Michael Jordan, and Philipp
  Moritz.
\newblock Trust region policy optimization.
\newblock In {\em International conference on machine learning}, pages
  1889--1897, 2015.

\bibitem{schulman2017proximal}
John Schulman, Filip Wolski, Prafulla Dhariwal, Alec Radford, and Oleg Klimov.
\newblock Proximal policy optimization algorithms.
\newblock {\em arXiv preprint arXiv:1707.06347}, 2017.

\bibitem{shani2020adaptive}
Lior Shani, Yonathan Efroni, and Shie Mannor.
\newblock Adaptive trust region policy optimization: Global convergence and
  faster rates for regularized mdps.
\newblock In {\em Proceedings of the AAAI Conference on Artificial
  Intelligence}, volume~34, pages 5668--5675, 2020.

\bibitem{shen2019hessian}
Zebang Shen, Alejandro Ribeiro, Hamed Hassani, Hui Qian, and Chao Mi.
\newblock Hessian aided policy gradient.
\newblock In {\em International Conference on Machine Learning}, pages
  5729--5738, 2019.

\bibitem{sidford2018near}
Aaron Sidford, Mengdi Wang, Xian Wu, Lin~F Yang, and Yinyu Ye.
\newblock Near-optimal time and sample complexities for solving discounted
  markov decision process with a generative model.
\newblock {\em arXiv preprint arXiv:1806.01492}, 2018.

\bibitem{sidford2018variance}
Aaron Sidford, Mengdi Wang, Xian Wu, and Yinyu Ye.
\newblock Variance reduced value iteration and faster algorithms for solving
  markov decision processes.
\newblock In {\em Proceedings of the Twenty-Ninth Annual ACM-SIAM Symposium on
  Discrete Algorithms}, pages 770--787. SIAM, 2018.

\bibitem{sutton2018reinforcement}
Richard~S Sutton and Andrew~G Barto.
\newblock {\em Reinforcement learning: An introduction}.
\newblock MIT press, 2018.

\bibitem{sutton2000policy}
Richard~S Sutton, David~A McAllester, Satinder~P Singh, and Yishay Mansour.
\newblock Policy gradient methods for reinforcement learning with function
  approximation.
\newblock In {\em Advances in neural information processing systems}, pages
  1057--1063, 2000.

\bibitem{wang2017stochastic}
Mengdi Wang, Ethan~X Fang, and Han Liu.
\newblock Stochastic compositional gradient descent: algorithms for minimizing
  compositions of expected-value functions.
\newblock {\em Mathematical Programming}, 161(1-2):419--449, 2017.

\bibitem{wang2017accelerating}
Mengdi Wang, Ji~Liu, and Ethan~X Fang.
\newblock Accelerating stochastic composition optimization.
\newblock {\em The Journal of Machine Learning Research}, 18(1):3721--3743,
  2017.

\bibitem{williams1992simple}
Ronald~J Williams.
\newblock Simple statistical gradient-following algorithms for connectionist
  reinforcement learning.
\newblock {\em Machine learning}, 8(3-4):229--256, 1992.

\bibitem{xu2019sample}
Pan Xu, Felicia Gao, and Quanquan Gu.
\newblock Sample efficient policy gradient methods with recursive variance
  reduction.
\newblock {\em arXiv preprint arXiv:1909.08610}, 2019.

\bibitem{xu2020improved}
Pan Xu, Felicia Gao, and Quanquan Gu.
\newblock An improved convergence analysis of stochastic variance-reduced
  policy gradient.
\newblock In {\em Uncertainty in Artificial Intelligence}, pages 541--551.
  PMLR, 2020.

\bibitem{zhang2020cautious}
Junyu Zhang, Amrit~Singh Bedi, Mengdi Wang, and Alec Koppel.
\newblock Cautious reinforcement learning via distributional risk in the dual
  domain.
\newblock {\em arXiv preprint arXiv:2002.12475}, 2020.

\bibitem{zhang2020variational}
Junyu Zhang, Alec Koppel, Amrit~Singh Bedi, Csaba Szepesvari, and Mengdi Wang.
\newblock Variational policy gradient method for reinforcement learning with
  general utilities.
\newblock {\em arXiv preprint arXiv:2007.02151}, 2020.

\bibitem{zhang2019composite}
Junyu Zhang and Lin Xiao.
\newblock A composite randomized incremental gradient method.
\newblock In {\em International Conference on Machine Learning}, pages
  7454--7462, 2019.

\bibitem{zhang2019multi}
Junyu Zhang and Lin Xiao.
\newblock Multi-level composite stochastic optimization via nested variance
  reduction.
\newblock {\em arXiv preprint arXiv:1908.11468}, 2019.

\bibitem{zhang2019stochastic}
Junyu Zhang and Lin Xiao.
\newblock A stochastic composite gradient method with incremental variance
  reduction.
\newblock In {\em Advances in Neural Information Processing Systems}, pages
  9078--9088, 2019.

\bibitem{zhang2020sample}
Junzi Zhang, Jongho Kim, Brendan O'Donoghue, and Stephen Boyd.
\newblock Sample efficient reinforcement learning with reinforce.
\newblock {\em arXiv preprint arXiv:2010.11364}, 2020.

\bibitem{zhang2019global}
Kaiqing Zhang, Alec Koppel, Hao Zhu, and Tamer Ba{\c{s}}ar.
\newblock Global convergence of policy gradient methods to (almost) locally
  optimal policies.
\newblock {\em arXiv preprint arXiv:1906.08383}, 2019.

\bibitem{zhao2019stochastic}
Mingming Zhao, Yongfeng Li, and Zaiwen Wen.
\newblock A stochastic trust-region framework for policy optimization.
\newblock {\em arXiv preprint arXiv:1911.11640}, 2019.

\bibitem{zhao2012analysis}
Tingting Zhao, Hirotaka Hachiya, Gang Niu, and Masashi Sugiyama.
\newblock Analysis and improvement of policy gradient estimation.
\newblock {\em Neural Networks}, 26:118--129, 2012.

\end{thebibliography}
\newpage
\appendix

\section{Proof of Lemma \ref{lemma:importance}}
\begin{proof}
	Let us prove the arguments of this lemma one by one. \\
	{\bf Proof of (i).} Note that the policy $\pi_\theta$ is parameterized by $\pi_{\theta}(a|s) = \frac{\exp\{\psi(s,a;\theta)\}}{\sum_{a'}\exp\{\psi(s,a';\theta)\}}.$
	By direct computation, we have 
	$$\nabla_\theta \log \pi_\theta(a|s) = \nabla_\theta\psi(s,a;\theta) - \sum_{a'}\pi_\theta(a'|s)\cdot\nabla_\theta\psi(s,a';\theta)$$
	\begin{eqnarray*}
		\nabla_\theta^2 \log \pi_\theta(a|s) & = & \nabla_\theta^2\psi(s,a;\theta) + \sum_{a',a''}\pi_\theta(a'|s)\pi_\theta(a''|s)\cdot\nabla_\theta\psi(s,a';\theta)\nabla_\theta\psi(s,a'';\theta)^\top \\
		&&-  \sum_{a'}\pi_\theta(a'|s)\cdot\left(\nabla_\theta^2\psi(s,a';\theta)+\nabla_\theta\psi(s,a';\theta)\nabla_\theta\psi(s,a';\theta)^\top\right).
	\end{eqnarray*}
	Because $\|\nabla_\theta\psi(s,a;\theta)\|\leq\ell_\psi$ and $\|\nabla_\theta^2\psi(s,a;\theta)\|\leq L_\psi$ for $\forall s,a, \theta$, we have 
	$$\|\nabla_\theta \log \pi_\theta(a|s)\|\leq \ell_\psi + \sum_{a'}\pi_\theta(a'|s)\ell_\psi = 2\ell_\psi,$$
	$$\|\nabla_\theta^2 \log \pi_\theta(a|s)\| \leq L_\psi + \sum_{a',a''}\pi_\theta(a'|s)\pi_\theta(a''|s)\cdot\ell_\psi^2 +   \sum_{a'}\pi_\theta(a'|s)\cdot\left(L_\psi+\ell^2_\psi\right) = 2(\ell_\psi^2+L_\psi).$$
	Next, for  $\|\nabla_\theta F(\lambda(\theta))\|$, by the chain rule and policy gradient \eqref{defn:PGT}, it holds that 
	\begin{eqnarray}
		\|\nabla_\theta F(\lambda(\theta))\|
		& = & \Big\|\EE\Big[\sum_{t=0}^{+\infty}\gamma^t\cdot\frac{\partial F(\lambda(\theta))}{\partial\lambda_{s_ta_t}}\cdot\Big(\sum_{t'=0}^t\nabla_{\theta}\log \pi_{\theta}(a_{t'}|s_{t'})\Big)\Big]\Big\|\nonumber\\
		& \leq & \EE\Big[\sum_{t=0}^{+\infty}\gamma^t\cdot \|\nabla_\lambda F(\lambda(\theta))\|_\infty \cdot\Big\|\Big(\sum_{t'=0}^t\nabla_{\theta}\log \pi_{\theta}(a_{t'}|s_{t'})\Big)\Big\|\Big]\nonumber\\
		& \leq & \sum_{t=0}^{+\infty}\gamma^t\ \cdot2(t+1)\ell_\psi \ell_{\lambda,\infty}\nonumber\\
		& \leq & \frac{2\ell_\psi\cdot \ell_{\lambda,\infty}}{(1-\gamma)^2}.
	\end{eqnarray}
	{\bf Proof of (ii).} Define $d(\theta,\theta'):=\|\lambda(\theta) - \lambda(\theta')\|_1$. Then 
	$$\nabla_\theta d(\theta,\theta') = \sum_{s,a}\sign(\lambda^{\pi_\theta}(s,a) - \lambda^{\pi_{\theta'}}(s,a))\cdot\nabla_\theta\lambda^{\pi_\theta}(s,a),$$
	where $\sign(x):= 1$ if $x\geq0$ and $\sign(x):=-1$ if $x<0$. Let ${\bf e}_{sa}$ be the vector with the $(s,a)$-th entry equal to 1 while other entries equal to 0. Then $\lambda^{\pi_\theta}(s,a) = \langle\lambda^{\pi_\theta}, {\bf e}_{sa}\rangle$ equals the cumulative sum of rewards with reward function being ${\bf e}_{sa}$. By Policy Gradient Theorem \cite{sutton2000policy}, we have 
	\begin{eqnarray*}
		\|\nabla_\theta d(\theta,\theta')\| & = & \big\|\sum_{s,a}\sign(\lambda^{\pi_\theta}(s,a) - \lambda^{\pi_{\theta'}}(s,a))\cdot\nabla_\theta\lambda^{\pi_\theta}(s,a)\big\|\\
		& \leq & \sum_{s,a}\|\nabla_\theta\lambda^{\pi_\theta}(s,a)\|\\
		& \overset{(a)}{=} & \sum_{s,a}\Big\|\EE\Big[\sum_{t=0}^{+\infty}\gamma^t\cdot {\bf e}_{sa}(s_t,a_t)\cdot\Big(\sum_{t'=0}^t\nabla_{\theta}\log \pi_{\theta}(a_{t'}|s_{t'})\Big)\Big]\Big\|\\
		& \overset{(b)}{\leq} & \sum_{s,a}\EE\Big[\sum_{t=0}^{+\infty}\gamma^t\cdot {\bf e}_{sa}(s_t,a_t)\cdot\Big\|\sum_{t'=0}^t\nabla_{\theta}\log \pi_{\theta}(a_{t'}|s_{t'})\Big\|\Big]\\
		& \overset{(c)}{\leq} & 2\ell_\psi\cdot\EE\Big[\sum_{t=0}^{+\infty}\gamma^t\cdot (t+1)\cdot\sum_{s,a}{\bf e}_{sa}(s_t,a_t)\Big]\\
		& \overset{(d)}{=} & 2\ell_\psi\cdot\sum_{t=0}^{+\infty}\gamma^t\cdot (t+1)\\
		& = & \frac{2\ell_\psi}{(1-\gamma)^2}.
	\end{eqnarray*}
	In the above arguments,  (a) is due to \eqref{defn:PGT}. (b) is because $\|\EE[X]\|\leq \EE[\|X\|]$ for any random vector $X$. (c) is due to (i) of Lemma \ref{lemma:importance}. (d) is because  $\sum_{sa}{\bf e}_{sa}(s',a')\equiv1$. As a result, 
	$$d(\theta,\theta') \leq d(\theta',\theta') + \frac{2\ell_\psi}{(1-\gamma)^2}\|\theta-\theta'\| = \frac{2\ell_\psi}{(1-\gamma)^2}\|\theta-\theta'\|.$$
	This completes the proof of (ii) of Lemma \ref{lemma:importance}.	\\
	{\bf Proof of (iii).} By the chain rule, $\nabla_\theta F(\lambda(\theta)) = [\nabla_\theta \lambda(\theta)]^\top\nabla_\lambda F(\lambda(\theta))$. Therefore, 
	\begin{eqnarray}
		\label{lm:importance-1}
		&&\|\nabla_\theta F(\lambda(\theta_1)) - \nabla_\theta F(\lambda(\theta_2))\|\\
		& = & \|[\nabla_\theta \lambda(\theta_1)]^\top\nabla_\lambda F(\lambda(\theta_1)) - [\nabla_\theta \lambda(\theta_2)]^\top\nabla_\lambda F(\lambda(\theta_2))\|\nonumber\\
		& \leq & \underbrace{\|[\nabla_\theta \lambda(\theta_1)]^\top(\nabla_\lambda F(\lambda(\theta_1)) - \nabla_\lambda F(\lambda(\theta_2)))\|}_{T_1} + \underbrace{\|[\nabla_\theta \lambda(\theta_1) - \nabla_\theta \lambda(\theta_2)]^\top\nabla_\lambda F(\lambda(\theta_2))\|}_{T_2}\nonumber.
	\end{eqnarray}
	For the term $T_1$, by \eqref{defn:PGT}, we have 
	\begin{eqnarray}
		\label{lm:importance-2}
		T_1& = &\big\|[\nabla_\theta \lambda(\theta_1)]^\top(\nabla_\lambda F(\lambda(\theta_1)) - \nabla_\lambda F(\lambda(\theta_2)))\big\|\\
		& = & \Big\|\EE\Big[\sum_{t=0}^{+\infty}\gamma^t\cdot \Big(\frac{\partial F(\lambda(\theta_1))}{\partial\lambda_{s_ta_t}} - \frac{\partial F(\lambda(\theta_2))}{\partial\lambda_{s_ta_t}}\Big)\cdot\Big(\sum_{t'=0}^t\nabla_{\theta}\log \pi_{\theta}(a_{t'}|s_{t'})\Big)\Big]\Big\|\nonumber\\
		& \leq & \EE\Big[\sum_{t=0}^{+\infty}\gamma^t\cdot \|\nabla_\lambda F(\lambda(\theta_1)) - \nabla_\lambda F(\lambda(\theta_2))\|_\infty \cdot\Big\|\Big(\sum_{t'=0}^t\nabla_{\theta}\log \pi_{\theta}(a_{t'}|s_{t'})\Big)\Big\|\Big]\nonumber\\
		& \overset{(a)}{\leq} & \sum_{t=0}^{+\infty}\gamma^t\ \cdot2(t+1)\ell_\psi \cdot L_{\lambda,\infty}\|\lambda(\theta_1) - \lambda(\theta_2)\|_1\nonumber\\
		& \overset{(b)}{\leq} & \frac{4\ell_\psi^2\cdot L_{\lambda,\infty}}{(1-\gamma)^4}\cdot\|\theta_1-\theta_2\|.\nonumber
	\end{eqnarray}
	In the above arguments, (a) is due to Assumption \ref{assumption:F-Lip} and (b) is because $\sum_{t=0}^{+\infty}\gamma^t\ \cdot(t+1) = \frac{1}{(1-\gamma)^2}$.
	For the term $T_2$, denote $r = \nabla_\lambda F(\lambda(\theta_2))$, $Q^{\pi_\theta}(s,a)$ be the Q-function for the discounted MDP with reward function $r$. We also define $\mu_{sa}^{\pi_\theta}$ as the occupancy measure with initial state distribution $p(\cdot|s,a)$:
	$$\mu_{sa}^{\pi_\theta}(s',a'):=\sum_{t=0}^{+\infty}\mathbb{P}\left(s_t = s',a_t = a' \,|\, \pi_\theta, s_0\sim p(\cdot|s,a)\right).$$
	Note that (ii) of Lemma \ref{lemma:importance} does not rely on the initial state distribution, therefore, $\|\mu_{sa}^{\pi_\theta}-\mu_{sa}^{\pi_{\theta'}}\|_1\leq\frac{2\ell_\psi}{(1-\gamma)^2}\cdot\|\theta_1-\theta_2\|$ still holds for any $s,a$. Therefore, by Policy Gradient Theorem \eqref{defn:PGT} and its equivalent form provided in \cite{sutton2000policy}, we have
	\begin{eqnarray}
		\label{lm:importance-3}
		T_2& = &\|[\nabla_\theta \lambda(\theta_1) - \nabla_\theta \lambda(\theta_2)]^\top\nabla_\lambda F(\lambda(\theta_2))\|\\
		& \overset{(a)}{=} & \big\|\sum_{sa}\lambda^{\pi_{\theta_1}}(s,a)Q^{\pi_{\theta_1}}(s,a)\nabla_{\theta}\log\pi_{\theta_1}(a|s) - \sum_{sa}\lambda^{\pi_{\theta_2}}(s,a)Q^{\pi_{\theta_2}}(s,a)\nabla_{\theta}\log\pi_{\theta_2}(a|s)\big\|\nonumber\\
		& \leq &\sum_{sa}\big|\lambda^{\pi_{\theta_1}}(s,a)\!-\!\lambda^{\pi_{\theta_2}}(s,a)\big|\!\cdot\! |Q^{\pi_{\theta_1}}(s,a)|\!\cdot\!\|\nabla_{\theta}\log\pi_{\theta_1}(a|s)\| \nonumber\\
		& & \!+\! \sum_{sa}\lambda^{\pi_{\theta_2}}(s,a)\!\cdot\!\big| Q^{\pi_{\theta_1}}(s,a)\!-\!Q^{\pi_{\theta_2}}(s,a)\big|\!\cdot\!\|\nabla_{\theta}\log\pi_{\theta_1}(a|s)\|\nonumber\\
		& & \!+\! \sum_{sa}\lambda^{\pi_{\theta_2}}(s,a)\!\cdot\!\big| Q^{\pi_{\theta_2}}(s,a)\big|\!\cdot\!\|\nabla_{\theta}\log\pi_{\theta_1}(a|s)-\nabla_{\theta}\log\pi_{\theta_2}(a|s)\|\nonumber\\
		& \overset{(b)}{\leq} & \frac{2\ell_\psi\!\cdot\!\ell_{\lambda,\infty}}{1-\gamma}\|\lambda^{\pi_{\theta_1}}-\lambda^{\pi_{\theta_2}}\|_1 \!+\! \frac{2\ell_\psi}{1-\gamma}\max_{sa}|Q^{\pi_{\theta_1}}(s,a)\!-\!Q^{\pi_{\theta_2}}(s,a)| \nonumber\\
		&&\!+\! \frac{\ell_{\lambda,\infty}}{(1-\gamma)^2}\max_{sa}\|\nabla_{\theta}\log\pi_{\theta_1}(a|s)-\nabla_{\theta}\log\pi_{\theta_2}(a|s)\|\nonumber
	\end{eqnarray}
	In the above argument, the $Q^{\pi_\theta}$ in (a) denotes the Q-function of a discounted MDP with reward function $r$ and policy $\pi_\theta$. (b) is because $\|\nabla_{\theta}\log\pi_{\theta}(a|s)\|\leq 2\ell_\psi$, $\sum_{sa}\lambda^{\pi_{\theta}}(s,a) = (1-\gamma)^{-1}$ and $\!Q^{\pi_{\theta}}(s,a)  = \frac{\|\nabla_{\lambda} F(\lambda(\theta))\|_\infty}{1-\gamma} = \frac{\ell_{\lambda,\infty}}{1-\gamma}$.  Therefore, we have
	\begin{eqnarray}
		T_2& \overset{(a)}{\leq} & \frac{4\ell_{\lambda,\infty}\cdot\ell_\psi^2}{(1-\gamma)^3}\cdot\|\theta_1-\theta_2\| + \frac{2\ell_{\psi}}{1-\gamma}\max_{sa}|\langle r, \mu_{sa}^{\pi_{\theta_1}} - \mu_{sa}^{\pi_{\theta_2}}\rangle| \!+\! \frac{2\ell_{\lambda,\infty}(L_\psi+\ell_{\psi}^2)}{(1-\gamma)^2}\|\theta_1-\theta_2\|\nonumber\\
		& \leq & \frac{4\ell_{\lambda,\infty}\cdot\ell_\psi^2}{(1-\gamma)^3}\cdot\|\theta_1-\theta_2\| + \frac{2\ell_{\psi}\cdot\|r\|_\infty}{1-\gamma}\|\mu_{sa}^{\pi_{\theta_1}} - \mu_{sa}^{\pi_{\theta_2}}\|_1 \!+\! \frac{2\ell_{\lambda,\infty}(L_\psi+\ell_{\psi}^2)}{(1-\gamma)^2}\|\theta_1-\theta_2\|\nonumber\\
		&\overset{(b)}{\leq}&\Big(\frac{8\ell_{\lambda,\infty}\cdot\ell_\psi^2}{(1-\gamma)^3}+\frac{2\ell_{\lambda,\infty}\cdot(L_\psi+\ell_{\psi}^2)}{(1-\gamma)^2}\Big)\cdot\|\theta_1-\theta_2\|  \nonumber,
	\end{eqnarray}
	where (a) is because $Q^{\pi_\theta}(s,a) = r(s,a) + \gamma\cdot\langle r,\mu_{sa}^{\pi_\theta}\rangle$ and (i) \& (ii) of Lemma \ref{lemma:importance}; (b) is due to applying (ii) of Lemma \ref{lemma:importance} to $\mu_{sa}^{\pi_{\theta}}$.
	Now combining \eqref{lm:importance-1}, \eqref{lm:importance-2} and \eqref{lm:importance-3} proves (iii) of Lemma \ref{lemma:importance}.
\end{proof}

\section{Proof of Lemma \ref{lemma:optimality-measure}}
\begin{proof}
	For the ease of notation, let us define $\mathcal{E}$ as the event when $\eta\|\nabla_\theta F(\lambda(\theta))\|<\delta$, and denote $\mathcal{E}^c$ as the complement of the event $\mathcal{E}$. By definition of $\cG_\eta(\theta)$, we have 
	\begin{eqnarray}
		\EE\big[\|\cG_\eta(\theta)\|\big]
		&=& \mathbb{P}\big(\mathcal{E}\big)\cdot\EE\Big[\|\cG_\eta(\theta)\|\,\Big|\,\mathcal{E}\Big] + \mathbb{P}\big(\mathcal{E}^c\big)\cdot\EE\Big[\|\cG_\eta(\theta)\|\,\Big|\,\mathcal{E}^c\Big]\nonumber\\
		& = & \mathbb{P}\big(\mathcal{E}\big)\cdot\EE\Big[\|\nabla_\theta F(\lambda(\theta))\|\,\Big|\,\mathcal{E}\Big]+\mathbb{P}\big(\mathcal{E}^c\big)\cdot\frac{\delta}{\eta}\nonumber\\
		&\leq & \epsilon\nonumber
	\end{eqnarray} 
	This indicates that 
	$$\mathbb{P}\big(\mathcal{E}\big)\cdot\EE\Big[\|\nabla_\theta F(\lambda(\theta))\|\,\Big|\,\mathcal{E}\Big]\leq\epsilon\quad\mbox{and}\quad\mathbb{P}\big(\mathcal{E}^c\big)\leq \frac{\eta\epsilon}{\delta}.$$
	Note that Lemma \ref{lemma:importance} indicates that $\|\nabla_\theta F(\lambda(\theta))\|\leq \frac{2\ell_\psi\cdot\ell_{\lambda,\infty}}{(1-\gamma)^2}$, combined with the above inequalities yields
	\begin{eqnarray}
		&&\EE[\|\nabla_\theta F(\lambda(\theta))\|] \nonumber\\
		& = & \mathbb{P}\big(\mathcal{E}\big)\cdot\EE\Big[\|\nabla_\theta F(\lambda(\theta))\|\,\Big|\,\mathcal{E}\Big] + \mathbb{P}\big(\mathcal{E}^c\big)\cdot\EE\Big[\|\nabla_\theta F(\lambda(\theta))\|\,\Big|\,\mathcal{E}^c\Big]\nonumber\\
		& \leq & \epsilon + \frac{\eta\epsilon}{\delta}\cdot\frac{2\ell_\psi\cdot\ell_{\lambda,\infty}}{(1-\gamma)^2}\nonumber\\
		& = & \left(1+ \frac{\eta}{\delta}\cdot\frac{2\ell_\psi\cdot\ell_{\lambda,\infty}}{(1-\gamma)^2}\right)\cdot\epsilon\nonumber.
	\end{eqnarray}
	This completes the proof. 
\end{proof}

\section{Proof of Lemma \ref{lemma:ascent-ncvx}}
To prove this lemma, let us first provide a supporting lemma. 
\begin{lemma}
	For Algorithm \ref{alg:TSIVR-PG} and any iterates $\theta_j^i$ and $\theta_{j+1}^i$, it holds that
	$$\|\cG_\eta(\theta_j^i)\|^2\leq2\eta^{-2}\!\cdot\!\|\theta_{j\!+\!1}^i \!-\! \theta_j^i\|^2 \!+\! 2\!\cdot\!\|g_j^i- \nabla_\theta F(\lambda(\theta_j^i))\|^2.$$
\end{lemma}
\begin{proof}
	For Algorithm \ref{alg:TSIVR-PG}, the truncated gradient update of the iterates can also be written as a gradient projection step:
	$\theta_{j+1}^i = \mathbf{Proj}_{B(\theta_j^i,\delta)}\left(\theta_j^i + \eta\cdot g_j^i\right)$. 
	Denote 
	$$\hat\theta_{j+1}^i = \mathbf{Proj}_{B(\theta_j^i,\delta)}\left(\theta_j^i + \eta\cdot \nabla_\theta F(\lambda(\theta_j^i))\right).$$
	Then by Cauchy's inequality and the non-expansiveness of the projection operator yields 
	\begin{eqnarray}
		&&\|\cG_\eta(\theta_j^i)\|^2\nonumber\\
		& = & \eta^{-2}\cdot\|\hat\theta_{j+1}^i - \theta_j^i\|^2 \nonumber\\
		& \leq & 2\eta^{-2}\cdot\|\theta_{j+1}^i - \theta_j^i\|^2 + 2\eta^{-2}\cdot\|\hat\theta_{j+1}^i - \theta_{j+1}^i\|^2\nonumber\\
		& = & 2\eta^{-2}\!\cdot\!\|\theta_{j\!+\!1}^i \!-\! \theta_j^i\|^2 \!+\! 2\eta^{-2}\!\cdot\!\|\mathbf{Proj}_{B(\theta_j^i,\delta)}\left(\theta_j^i \!+\! \eta\cdot g_j^i\right) \!-\! \mathbf{Proj}_{B(\theta_j^i,\delta)}\left(\theta_j^i \!+\! \eta\cdot \nabla_\theta F(\lambda(\theta_j^i))\right)\!\|^2\nonumber\\
		& \leq & 2\eta^{-2}\!\cdot\!\|\theta_{j\!+\!1}^i \!-\! \theta_j^i\|^2 \!+\! 2\!\cdot\!\|g_j^i- \nabla_\theta F(\lambda(\theta_j^i))\|^2.\nonumber
	\end{eqnarray}
\end{proof}
Now we are ready to provide the proof of Lemma \ref{lemma:ascent-ncvx}.
\begin{proof}
	By the $L_\theta$-smoothness of the objective function, we have 
	\begin{eqnarray}
		&&F(\lambda(\theta_{j+1}^i)) \nonumber\\
		& \geq & F(\lambda(\theta_{j}^i)) + \langle\nabla_\theta F(\lambda(\theta_j^i)),\theta_{j+1}^i-\theta_j^i\rangle - \frac{L_\theta}{2}\|\theta_{j+1}^i-\theta_j^i\|^2\nonumber\\
		& = & F(\lambda(\theta_{j}^i)) + \langle g_j^i,\theta_{j+1}^i-\theta_j^i\rangle - \frac{1}{2\eta}\|\theta_{j+1}^i-\theta_j^i\|^2 + \langle\nabla_\theta F(\lambda(\theta_j^i))-g_j^i,\theta_{j+1}^i-\theta_j^i\rangle\nonumber\\
		& & + \Big(\frac{1}{2\eta}-\frac{L_\theta}{2}\Big)\|\theta_{j+1}^i-\theta_j^i\|^2\nonumber\\
		& \overset{(i)}{\geq} & F(\lambda(\theta_{j}^i)) + \langle\nabla_\theta F(\lambda(\theta_j^i))-g_j^i,\theta_{j+1}^i-\theta_j^i\rangle  + \Big(\frac{1}{\eta}-\frac{L_\theta}{2}\Big)\|\theta_{j+1}^i-\theta_j^i\|^2\nonumber\\
		& \overset{(ii)}{\geq} & F(\lambda(\theta_{j}^i)) + \frac{\eta}{4}\|\cG_\eta(\theta_j^i)\|^2 + \Big(\frac{1}{2\eta}-L_\theta\Big)\|\theta_{j+1}^i-\theta_j^i\|^2 - \Big(\frac{\eta}{2} + \frac{1}{2L_\theta}\Big )\|\nabla_\theta F(\lambda(\theta_j^i))-g_j^i\|^2\nonumber
	\end{eqnarray}
	where (i) is due to \eqref{defn:TR}; (ii) is due to $$\langle\nabla_\theta F(\lambda(\theta_j^i))-g_j^i,\theta_{j+1}^i-\theta_j^i\rangle\geq - \frac{1}{2L_\theta}\|\nabla_\theta F(\lambda(\theta_j^i))-g_j^i\|^2 - \frac{L_\theta}{2}\|\theta_{j+1}^i-\theta_j^i\|^2$$
	and adding 
	$$0\geq\frac{\eta}{4}\|\cG_\eta(\theta_j^i)\|^2-\frac{1}{2\eta}\!\cdot\!\|\theta_{j\!+\!1}^i \!-\! \theta_j^i\|^2 \!-\! \frac{\eta}{2}\!\cdot\!\|g_j^i- \nabla_\theta F(\lambda(\theta_j^i))\|^2$$
	to both sides of (i). Taking expectation on both sides of the above inequality proves the lemma. 
\end{proof}

\section{Proof of Lemma \ref{lemma:Weight-Var0}}
\begin{proof}
	Due to the parameterization form \eqref{defn:policy-para}, for any $\theta_1,\theta_2$ and any state-action pair $(s,a)$, we have 
	\begin{eqnarray*}
		\label{lm:importance-4}
		\frac{\pi_{\theta_2}(a|s)}{\pi_{\theta_1}(a|s)} & = & \frac{\exp\{\psi(s,a;\theta_2)\}}{\exp\{\psi(s,a;\theta_1)\}}\cdot\frac{\sum_{a'}\exp\{\psi(s,a';\theta_1)\}}{\sum_{a'}\exp\{\psi(s,a';\theta_2)\}}\\
		& \leq & \exp\{\psi(s,a;\theta_2)-\psi(s,a;\theta_1)\}\cdot\max_{a'}\exp\{\psi(s,a';\theta_1)-\psi(s,a';\theta_2)\}\\
		& \leq & \exp\{2\ell_\psi\cdot\|\theta_1-\theta_2\|\}
	\end{eqnarray*}
	As a result, by definition, for any $t\in\{0,1,...,H-1\}$, the importance sampling weight is 
	$$\omega_t(\tau|\theta_1,\theta_2) = \prod_{t'=0}^{t}\frac{\pi_{\theta_2}(a_{t'}|s_{t'})}{\pi_{\theta_1}(a_{t'}|s_{t'})}\leq \exp\{2(t+1)\ell_\psi\cdot\|\theta_1-\theta_2\|\}.$$
\end{proof}

\section{A few supporting lemmas}
\label{appdx:prop-weighted-sampling}
First, we would like to introduce the lemma that describes the properties of the off-policy sampling estimators. For the ease of discussion, let us define the occupancy measure of a $H$-horizon truncated trajectory as 
\begin{eqnarray}
	\label{defn:occupancy-trun}
	\lambda_H(s,a;\theta) := \sum_{t=0}^{H-1}\gamma^t\cdot\mathbb{P}\Big(s_t = s, a_t = a \,\big|\, \pi_\theta, s_0\sim\xi\Big),
\end{eqnarray}
for $\forall (s,a)\in\cS\times\cA$. Then for any vector $r$, we have 
\begin{align}
	\label{lm:PG-Jacob-trun-1}
	[\nabla_\theta \lambda_H(\theta)]^{\top}r = \EE\bigg[\sum_{t=0}^{H-1}\gamma^t\cdot r(s_t,a_t)\cdot\Big(\sum_{t'=0}^{t}\nabla_{\theta}\log \pi_{\theta}(a_{t'} | s_{t'})\Big)\Big| \pi_\theta, s_0\sim\xi\bigg].
\end{align}

Next, let us focus on the off-policy estimators. 
\begin{proposition} 
	\label{proposition:weighted-sampling} 
	Let $\tau = \{s_0, a_0, s_1, a_1, \cdots, s_{H-1}, a_{H-1}\}$ be sampled from the behavioral policy $\pi_{\theta_1}$. Then for the target policy $\pi_{\theta_2}$, it holds that
	\begin{eqnarray*}
		\EE_{\tau\sim\pi_{\theta_1}}\left[\widehat\lambda_\omega(\tau|\theta_1,\theta_2)\right] = \lambda_H(\theta_2)\qquad\mbox{and}\qquad \EE_{\tau\sim\pi_{\theta_1}}\left[\widehat{g}_\omega(\tau|\theta_1,\theta_2,r)\right] = [\nabla_\theta\lambda_H(\theta_2)]^\top r.
	\end{eqnarray*}
	The above equations also hold in case $\theta_1=\theta_2$. 
\end{proposition}
\begin{proof}
	First, let us define $\tau_t = \{s_0,a_0,...,s_t,a_t\}$ as the truncated trajectory of $\tau$ with length $t+1$. Then we can write 
	$$p(\tau_t|\pi_{\theta}) = \xi(s_0)\pi_{\theta}(a_0|s_0)\cdot\prod_{t'=1}^{t}P(s_{t'}|s_{t'-1},a_{t'-1})\pi_{\theta}(a_{t'}|s_{t'}).$$
	Then we also have $\omega_t(\tau|\theta_1,\theta_2) = \frac{p(\tau_t|\pi_{\theta_2})}{p(\tau_t|\pi_{\theta_1})}$.
	Consequently, we have 
	\begin{eqnarray}
		&&\EE_{\tau\sim\pi_{\theta_1}}\left[\widehat\lambda_\omega(\tau|\theta_1,\theta_2)\right]\nonumber\\
		& = & \sum_{\tau} p(\tau|\pi_{\theta_1})\cdot\bigg(\sum_{t=0}^{H-1}\gamma^t\cdot\omega_t(\tau|\theta_1,\theta_2)\cdot\be_{s_ta_t}\bigg)\nonumber\\
		& = & \sum_{\tau} p(\tau|\pi_{\theta_1})\cdot\bigg(\sum_{t=0}^{H-1}\gamma^t\cdot\frac{p(\tau_t|\pi_{\theta_2})}{p(\tau_t|\pi_{\theta_1})}\cdot\be_{s_ta_t}\bigg)\nonumber\\
		& = & \sum_{\tau} \sum_{t=0}^{H-1}\gamma^t\cdot p(\tau_t|\pi_{\theta_2})\left(\prod_{t'=t+1}^{H-1}P(s_{t'}|s_{t'-1},a_{t'-1})\pi_{\theta_1}(a_{t'}|s_{t'})\right)\cdot\be_{s_ta_t}\nonumber\\
		& = & \sum_{\tau_t} \gamma^t\cdot p(\tau_t|\pi_{\theta_2})\sum_{\{s_{t'},a_{t'}\}_{t'=t+1}^{H-1}}\left(\prod_{t'=t+1}^{H-1}P(s_{t'}|s_{t'-1},a_{t'-1})\pi_{\theta_1}(a_{t'}|s_{t'})\right)\cdot\be_{s_ta_t}\nonumber\\
		& \overset{(i)}{=} & \sum_{\tau_t} \gamma^t\cdot p(\tau_t|\pi_{\theta_2})\sum_{\{s_{t'},a_{t'}\}_{t'=t+1}^{H-1}}\left(\prod_{t'=t+1}^{H-1}P(s_{t'}|s_{t'-1},a_{t'-1})\pi_{\theta_2}(a_{t'}|s_{t'})\right)\cdot\be_{s_ta_t}\nonumber\\
		& = & \sum_{\tau} \sum_{t=0}^{H-1}\gamma^t\cdot p(\tau_t|\pi_{\theta_2})\left(\prod_{t'=t+1}^{H-1}P(s_{t'}|s_{t'-1},a_{t'-1})\pi_{\theta_2}(a_{t'}|s_{t'})\right)\cdot\be_{s_ta_t}\nonumber\\
		& = & \sum_{\tau} p(\tau|\pi_{\theta_2})\cdot\bigg(\sum_{t=0}^{H-1}\gamma^t\cdot\be_{s_ta_t}\bigg)\nonumber\\
		& = & \lambda_H(\theta_2),\nonumber
	\end{eqnarray}
	where (i)  is due to the fact that $\sum_{\{s_{t'},a_{t'}\}_{t'=t+1}^{H-1}}\left(\prod_{t'=t+1}^{H-1}P(s_{t'}|s_{t'-1},a_{t'-1})\pi_{\theta}(a_{t'}|s_{t'})\right) \equiv 1$ for any policy $\pi_\theta$.	Similarly, we also have 
	$$\EE\left[\widehat{g}_{\omega}(\tau|\theta_1,\theta_2,r)\right] = [\nabla_\theta\lambda_H(\theta_2)]^\top r.$$
\end{proof}
Next, we introduce the Lipschitz continuity of these estimators.  
\begin{lemma}
	\label{lemma:Lipschitz-g}
	Let $\tau = \{s_0,a_0,...,s_{H-1},a_{H-1}\}$ be an arbitrary trajectory. Then the following inequalities hold true.
	\begin{itemize}
		\item[(i).] For $\forall\theta$, and $\forall r_1, r_2$, it holds that $\|\widehat g(\tau|\theta,r_1)-\widehat g(\tau|\theta,r_2)\|\leq \frac{2\ell_\psi}{(1-\gamma)^2}\cdot\|r_1-r_2\|_\infty.$
		\item[(ii).] For $\forall\theta_1,\theta_2$, and $\forall r$,  it holds that $
		\|\widehat g(\tau|\theta_1,r)-\widehat g(\tau|\theta_2,r)\|\leq \frac{2(\ell_\psi^2+L_\psi)\cdot\|r\|_\infty}{(1-\gamma)^2}\cdot\|\theta_1-\theta_2\|.$
	\end{itemize}
\end{lemma}
\begin{proof}
	For the first inequality, we have  
	\begin{eqnarray*} 
		\|\widehat g(\tau|\theta,r_1)-\widehat g(\tau|\theta,r_2)\| & = & \Big\| \sum_{t=0}^{H-1}\gamma^t\cdot \left(r_1(s_t,a_t)-r_2(s_t,a_t)\right)\cdot\Big(\sum_{t'=0}^t\nabla_\theta\log \pi_{\theta}(a_{t'}|s_{t'})\Big)\Big\|\nonumber\\
		& \leq & \sum_{t=0}^{H-1}\gamma^t(t+1)\cdot2\ell_\psi\cdot\|r_1-r_2\|_\infty\nonumber\\
		&\leq& \frac{2\ell_\psi\cdot\|r_1-r_2\|_\infty}{(1-\gamma)^2}.\nonumber
	\end{eqnarray*}
	For the second inequality, we have 
	\begin{eqnarray*} 
		\big\|\widehat g(\tau|\theta_1,r)-\widehat g(\tau|\theta_2,r)\big\| & = & \Big\|\! \sum_{t=0}^{H-1}\!\gamma^t\!\cdot\! r(s_t,a_t)\!\cdot\!\Big(\!\sum_{t'=0}^t\nabla_\theta\log \pi_{\theta_1}(a_{t'}|s_{t'})\!-\!\nabla_\theta\log \pi_{\theta_2}(a_{t'}|s_{t'})\!\Big)\!\Big\|\\
		& \leq & \sum_{t=0}^{H-1}\gamma^t(t+1)\cdot2(\ell_\psi^2+L_\psi)\cdot\|\theta_1-\theta_2\|\cdot\|r\|_\infty\\
		&\leq& \frac{2(\ell_\psi^2+L_\psi)\cdot\|r\|_\infty}{(1-\gamma)^2}\cdot\|\theta_1-\theta_2\|.
	\end{eqnarray*}
	Hence we complete the proof.
\end{proof}
Finally, we provide the following lemma to characterize the property of the truncated occupancy measure.
\begin{lemma}
	\label{lemma:Trun-Occupancy}
	For any $H$, the following inequality holds that
	\begin{eqnarray*}
		\big\|\nabla_{\theta} F(\lambda_H(\theta)) -\nabla_{\theta} F(\lambda(\theta))\big\|^2\leq \left(\frac{8\ell_\psi^2\cdot L_{\lambda}^2}{(1-\gamma)^6} + 16\ell_\psi^2\ell_{\lambda,\infty}^2\Big(\frac{(H+1)^2}{(1-\gamma)^2} + \frac{1}{(1-\gamma)^4}\Big) \right)\cdot\gamma^{2H}.
	\end{eqnarray*}
\end{lemma}
\begin{proof}
	First, by triangle inequality, we get 
	\begin{eqnarray}
		&&\big\|\nabla_{\theta} F(\lambda_H(\theta)) -\nabla_{\theta} F(\lambda(\theta))\big\|^2\nonumber\\
		& = & \big\|[\nabla_{\theta} \lambda_H(\theta)]^\top \nabla_{\lambda}F(\lambda_H(\theta)) -[\nabla_{\theta} \lambda(\theta)]^\top \nabla_{\lambda}F(\lambda(\theta))\big\|^2\nonumber\\
		& \leq & 2\|[\nabla_{\theta}\lambda_H(\theta)]^\top(\nabla_{\lambda}F(\lambda_H(\theta)) - \nabla_{\lambda}F(\lambda(\theta)))\|^2 \nonumber\\
		&&+ 2\|([\nabla_{\theta}\lambda_H(\theta)]^\top-[\nabla_{\theta} \lambda(\theta)]^\top) \nabla_{\lambda}F(\lambda(\theta))\|^2\nonumber.
	\end{eqnarray}
	For the first term, we have 
	\begin{eqnarray}
		&&\|[\nabla_{\theta}\lambda_H(\theta_j^i)]^\top(\nabla_{\lambda}F(\lambda_H(\theta_{j}^i)) - \nabla_{\lambda}F(\lambda(\theta_{j}^i)))\|^2\nonumber\\
		& \overset{(a)}{\leq} & \frac{4\ell_\psi^2}{(1-\gamma)^4}\|\nabla_{\lambda}F(\lambda_H(\theta_{j}^i)) - \nabla_{\lambda}F(\lambda(\theta_{j}^i))\|_\infty^2\nonumber\\
		& \overset{(b)}{\leq} & \frac{4\ell_\psi^2\cdot L_{\lambda,\infty}^2}{(1-\gamma)^4}\|\lambda_H(\theta_{j}^i) - \lambda(\theta_{j}^i)\|_1^2\nonumber\\
		& \leq & \frac{4\ell_\psi^2\cdot L_{\lambda,\infty}^2}{(1-\gamma)^6}\cdot\gamma^{2H}.\nonumber
	\end{eqnarray}
	In the above argument, (a) follows the argument of \eqref{lm:importance-2} and (b) is due to Assumption \ref{assumption:F-Lip}. For the second term, we have 
	\begin{eqnarray}
		&&\left\|[\nabla_\theta\lambda_H(\theta) - \nabla_\theta\lambda(\theta)]^\top\nabla_{\lambda} F(\lambda(\theta_j^i))\right\|^2\nonumber\\
		& = & \left\|\EE\left[\sum_{t=H}^{+\infty}\gamma^t\cdot\nabla_{\lambda_{s_ta_t}}F(\lambda(\theta_j^i))\cdot\Big(\sum_{t'=0}^{t}\nabla_\theta\log \pi_{\theta}(a_{t'} | s_{t'})\Big)\right]\right\|^2\nonumber\\
		&\leq& \left(\sum_{t=H}^{+\infty}\gamma^t\cdot2(t+1)\ell_\psi\cdot\ell_{\lambda,\infty}\right)^2\nonumber\\
		& \leq & 8\ell_\psi^2\ell_{\lambda,\infty}^2\left(\frac{(H+1)^2}{(1-\gamma)^2} + \frac{1}{(1-\gamma)^4}\right)\cdot\gamma^{2H}.\nonumber
	\end{eqnarray}	 
	Combining the above inequalities proves the lemma.
\end{proof}

\section{Proof of Lemma \ref{lemma:SIVR-variance}}
\subsection{Bounding the variance of $\lambda^i_j$}
To prove this lemma, the first step is to bound the variance of the occupancy estimator $\lambda_j^i$, which is shown in the following supporting lemma.
\begin{lemma}
	\label{lemma:SIVR-occupancy}
	For the occupancy estimators $\lambda_{j}^i$, we have 
	\begin{align} 
		\label{lm:occupancy-Var}
		\EE\Big[\big\|\lambda_{j}^i - \lambda_H(\theta_{j}^i)\big\|^{2}\Big] \leq \frac{1}{N(1-\gamma)^{2}}+ \frac{2H(8\ell_\psi^2+L_\psi)(e^{2H\ell_\psi\delta} + 1)}{(1-\gamma)^3\cdot B}\cdot\sum_{j'=1}^{j}\EE\Big[\|\theta_{j'-1}^i-\theta_{j'}^i\|^2\Big],\nonumber
	\end{align} 
	where $\lambda_H$ is defined in Appendix \ref{appdx:prop-weighted-sampling}.
\end{lemma}

\begin{proof}
	First, let us prove this lemma for the case $j = 0$. Note that $\lambda_{0}^i = \frac{1}{N}\sum_{\tau\in\cN_i}\widehat\lambda(\tau|\theta_{0}^i)$, by the independence of the trajectories and every $\tau\in\cN_i$ is sampled under $\pi_{\theta_0^i}$, we have 
	\begin{eqnarray*}
		\EE\big[\big\|\lambda_{0}^i - \lambda_H(\theta_{0}^i)\big\|^2\big] & = & \EE\Big[\Big\|\frac{1}{N}\sum_{\tau\in\cN_i}\widehat\lambda(\tau|\theta_{0}^i) - \lambda_H(\theta_{0}^i)\Big\|^2\Big] \\
		& \overset{(a)}{=} & \frac{1}{N}\cdot\EE\Big[\Big\|\widehat\lambda(\tau|\theta_{0}^i) - \lambda_H(\theta_{0}^i)\Big\|^2\Big]\\
		& \overset{(b)}{\leq} & \frac{1}{N}\cdot\EE\Big[\Big\|\widehat\lambda(\tau|\theta_{0}^i)\Big\|^2\Big]\\
		& \overset{(c)}{\leq} & (1-\gamma)^{-2}\cdot N^{-1},
	\end{eqnarray*}
	where (a) is due to $\lambda_H(\theta_{0}^i) = \EE\big[\widehat\lambda(\tau|\theta_{0}^i)\big]$, see Proposition \ref{proposition:weighted-sampling}, and the fact that the trajectories in $\cN_i$ are independently sampled; (b) is because $\mathrm{Var}(X)\leq \EE\big[\|X\|^2\big]$ for any random vector $X$; (c) is because $\big\|\widehat\lambda(\tau|\theta_{0}^i)\big\|\leq (1-\gamma)^{-1}$ w.p. 1. Next, let us prove the lemma for the case $j\geq1$, where each trajectory $\tau\in\cB_j^i$ is sampled under policy $\pi_{\theta_j^i}$.
	\begin{eqnarray}
		\label{lm:SIVR-occupancy-variance-pf-1}
		&&\EE\left[\|\lambda_{j}^i-\lambda_H(\theta_{j}^i)\|^2\right]\\
		& = & \EE\bigg[\Big\|\frac{1}{B}\sum_{\tau\in\cB_j^i}\left(\widehat\lambda(\tau|\theta_{j}^i) - \widehat\lambda_\omega\left(\tau|\theta_{j}^i,\theta_{j-1}^i\right) \right) + \lambda_{j-1}^i-\lambda_H(\theta_{j-1}^i) + \lambda_H(\theta_{j-1}^i)-\lambda_H(\theta_{j}^i) \Big\|^2\bigg]\nonumber\\
		& = & \EE\bigg[\Big\|\frac{1}{B}\sum_{\tau\in\cB_j^i}\!\left(\widehat\lambda(\tau|\theta_{j}^i) - \widehat\lambda_\omega\left(\tau|\theta_{j}^i,\theta_{j-1}^i\right)\!\right) + \lambda_H(\theta_{j-1}^i)-\lambda_H(\theta_{j}^i) \Big\|^2\bigg] \!\!+\! \EE\bigg[\Big\| \lambda_{j-1}^i-\lambda_H(\theta_{j-1}^i)\Big\|^2\bigg]\nonumber\\
		& & + 2\EE\bigg[\Big\langle\frac{1}{B}\sum_{\tau\in\cB_j^i}\left(\widehat\lambda(\tau|\theta_{j}^i) - \widehat\lambda_\omega\left(\tau|\theta_{j}^i,\theta_{j-1}^i\right) \right)  + \lambda_H(\theta_{j-1}^i)-\lambda_H(\theta_{j}^i),  \lambda_{j-1}^i-\lambda_H(\theta_{j-1}^i) \Big\rangle\bigg]\nonumber
	\end{eqnarray}
	Note that 
	$$\EE_{\tau\sim\pi_{\theta_j^i}}\bigg[\frac{1}{B}\sum_{\tau\in\cB_j^i}\widehat\lambda(\tau|\theta_{j}^i) - \lambda_H(\theta_{j}^i)\bigg] = \EE_{\tau\sim\pi_{\theta_j^i}}\bigg[\frac{1}{B}\sum_{\tau\in\cB_j^i}\widehat\lambda_\omega\left(\tau|\theta_{j}^i,\theta_{j-1}^i\right) - \lambda_H(\theta_{j-1}^i)\bigg] = 0.$$
	For the first term of \eqref{lm:SIVR-occupancy-variance-pf-1}, we have 
	\begin{eqnarray}
		&&\label{lm:SIVR-variance-pf-2}
		\EE\bigg[\Big\|\frac{1}{B}\sum_{\tau\in\cB_j^i}\left(\widehat\lambda(\tau|\theta_{j}^i) - \widehat\lambda_\omega\left(\tau|\theta_{j}^i,\theta_{j-1}^i\right) \right) + \lambda_H(\theta_{j-1}^i)-\lambda_H(\theta_{j}^i) \Big\|^2\bigg]\\
		& \overset{(a)}{=} & \frac{1}{B}\cdot\EE\bigg[\Big\|\widehat\lambda(\tau|\theta_{j}^i) - \widehat\lambda_\omega\left(\tau|\theta_{j}^i,\theta_{j-1}^i\right) + \lambda_H(\theta_{j-1}^i)-\lambda_H(\theta_{j}^i) \Big\|^2\bigg]\nonumber\\
		& \overset{(b)}{\leq} & \frac{1}{B}\cdot\EE\Big[\big\|\widehat\lambda(\tau|\theta_{j}^i) - \widehat\lambda_\omega\left(\tau|\theta_{j}^i,\theta_{j-1}^i\right)\big\|^2\Big]\nonumber\\
		& \overset{(c)}{=} & \frac{1}{B}\cdot\EE\bigg[\Big\|\sum_{t=0}^{H-1}\gamma^t\cdot\big(1-\omega_t(\tau|\theta_{j}^i,\theta_{j-1}^i)\big)\cdot\be_{s_ta_t}\Big\|^2\bigg]\nonumber 	\\
		& \overset{(d)}{\leq} & \frac{H}{B}\cdot\sum_{t=0}^{H-1}\gamma^t\cdot\text{Var}\big(\omega_t(\tau|\theta_{j}^i,\theta_{j-1}^i)\big)\nonumber\\
		&\overset{(e)}{\leq}& \frac{H\cdot\EE\left[\|\theta_{j-1}^i-\theta_{j}^i\|^2\right]}{B}\cdot\sum_{t=0}^{H-1}\gamma^t\cdot(t+1)\big(8(t+1)\ell_\psi^2 +2(\ell_\psi^2 + L_\psi)\big)(e^{2H\ell_\psi\delta} + 1)\nonumber\\
		& \leq & \frac{H\cdot\EE\left[\|\theta_{j-1}^i-\theta_{j}^i\|^2\right]}{B}\cdot\sum_{t=0}^{H-1}\gamma^t\cdot2(t+1)(t+2)\big(8\ell_\psi^2+L_\psi\big)(e^{2H\ell_\psi\delta} + 1)\nonumber\\
		& \leq & \frac{2H(8\ell_\psi^2+L_\psi)(e^{2H\ell_\psi\delta} + 1)}{(1-\gamma)^3\cdot B}\cdot\EE\left[\|\theta_{j-1}^i-\theta_{j}^i\|^2\right]\nonumber
	\end{eqnarray}
	where (a) is due to the unbiasedness of the difference estimator; (b) is because $\mathrm{Var}(X)\leq \EE[\|X\|^2]$ for any random vector $X$; (c) is due to the definition \eqref{defn:weight-lambda-estimator}; (d) utilizes the inequality that $(\sum_{h=1}^Hx_h)^2\leq H\sum_{h=1}^Hx_h^2$ and (e) is due to Lemma \ref{lemma:Weight-Var}. Now, substituting \eqref{lm:SIVR-variance-pf-2} into \eqref{lm:SIVR-occupancy-variance-pf-1} yields
	\begin{eqnarray}
		\EE\left[\|\lambda_{j}^i-\lambda_H(\theta_{j}^i)\|^2\right] \leq 
		\EE\left[\|\lambda_{j-1}^i-\lambda_H(\theta_{j-1}^i)\|^2\right] + \frac{2H(8\ell_\psi^2+L_\psi)(e^{2H\ell_\psi\delta} + 1)}{(1-\gamma)^3\cdot B}\cdot\EE\left[\|\theta_{j-1}^i-\theta_{j}^i\|^2\right]\nonumber.
	\end{eqnarray}
	Recursively summing the above inequalities up yields
	\begin{eqnarray*}
		\EE\left[\|\lambda_{j}^i-\lambda_H(\theta_{j}^i)\|^2\right] &\leq& 
		\EE\left[\|\lambda_{0}^i-\lambda_H(\theta_{0}^i)\|^2\right] + \frac{2H(8\ell_\psi^2+L_\psi)(e^{2H\ell_\psi\delta} + 1)}{(1-\gamma)^3\cdot B}\cdot\sum_{j'=1}^{j}\EE\left[\|\theta_{j-1}^i-\theta_{j}^i\|^2\right]\nonumber\\
		& \leq & (1-\gamma)^{-2}N^{-1} + \frac{2H(8\ell_\psi^2+L_\psi)(e^{2H\ell_\psi\delta} + 1)}{(1-\gamma)^3\cdot B}\cdot\sum_{j'=1}^{j}\EE\left[\|\theta_{j-1}^i-\theta_{j}^i\|^2\right].
	\end{eqnarray*} 
	Combining the case for $j=0$ and $j\geq1$ proves the lemma. 
\end{proof}

\subsection{Proof of Lemma \ref{lemma:SIVR-variance}}
First, we present a more detailed version of Lemma \ref{lemma:SIVR-variance}, which contains the expressions of the constants $C_i, i = 1,...,4$. 
\begin{lemma} 
	For the PG estimators $g_j^i$, we have  
	\begin{equation*}
		\EE\left[\Big\|g_j^i-\nabla_{\theta} F(\lambda(\theta_{j}^i))\Big\|^2\right] \leq \frac{C_1}{N} + C_2\gamma^{2H} + \frac{C_3}{B}\cdot\sum_{j'=1}^j\EE\left[\|\theta_{j'-1}^i-\theta_{j'}^i\|^2\right] + C_4\EE\left[\|\theta_{j-1}^i-\theta_{j}^i\|^2\right]
	\end{equation*}
	where the constants $C_1, C_2, C_3$ and $C_4$ are defined as
	$$C_1 = \frac{112\ell_\psi^2\cdot L_{\lambda}^2}{(1-\gamma)^6} + \frac{12\ell_{\lambda,\infty}^2}{(1-\gamma)^4} ,\qquad C_2 = \frac{32\ell_\psi^2\cdot L_{\lambda}^2}{(1-\gamma)^6} + 64\ell_\psi^2\ell_{\lambda,\infty}^2\Big(\frac{(H+1)^2}{(1-\gamma)^2} + \frac{1}{(1-\gamma)^4}\Big)$$
	$$C_3 = \frac{48(\ell_\psi\!+\!L_\psi)^2\ell_{\lambda,\infty}^2}{(1-\gamma)^4}\!+\!\frac{96H\ell_{\lambda,\infty}^2\big(8\ell_\psi^2\!+\! L_\psi\big)(e^{2H\ell_\psi\delta} \!+\! 1)}{(1-\gamma)^5}\!\cdot\!\left(12\ell_{\lambda,\infty}^2 \!+\! \frac{4L_{\lambda}^2}{3(1-\gamma)^2}\right) $$
	$$C_4 = \frac{32H\ell_\psi^2 L_{\lambda}^2(8\ell_\psi^2+L_\psi)(e^{2H\ell_\psi\delta} + 1)}{(1-\gamma)^7}.$$
\end{lemma}

\begin{proof}
	The proof of this lemma is very lengthy, we will separate it into several steps:\\
	\textbf{Step 1.} Show for any $j\geq1$ that 
	{\footnotesize\begin{eqnarray}
			\label{lm:SIVR-PG-variance-step1}
			&&\!\!\!\!\EE\left[\Big\|g_j^i-\left[\nabla_\theta\lambda_H(\theta_j^i)\right]^\top r_{j-1}^i\Big\|^2\right] - \EE\left[\Big\|g_0^i-\left[\nabla_\theta\lambda_H(\theta_{0}^i)\right]^\top r_{0}^i\Big\|^2\right]\\
			\!\!\!\!&\!\!\!\leq&\!\!\!\! 
			\left(\frac{12(\ell_\psi\!+\!L_\psi)^2\ell_{\lambda,\infty}^2}{(1-\gamma)^4}\!+\!\frac{24H\ell_{\lambda,\infty}^2\big(8\ell_\psi^2\!+\! L_\psi\big)(e^{2H\ell_\psi\delta} \!+\! 1)}{(1-\gamma)^5}\!\cdot\!\left(12\ell_{\lambda,\infty}^2 \!+\! \frac{L_{\lambda}^2}{(1-\gamma)^2}\right)\!\right)\!\cdot\!\frac{\sum_{j'=1}^j\EE\left[\|\theta_{j'}^i-\theta_{j'-1}^i\|^2\right]}{B} \nonumber.
	\end{eqnarray}}
	\textbf{Step 2.} Show for $j=0$ that 
	\begin{eqnarray}
		\label{lm:SIVR-PG-variance-step2}
		\EE\left[\big\|g_0^i-\left[\nabla_\theta\lambda_H(\theta_{0}^i)\right]^\top r_{0}^i\big\|^2\right]
		\leq \frac{24\ell_\psi^2\cdot L_{\lambda}^2}{(1-\gamma)^6\cdot N} + \frac{3\ell_{\lambda,\infty}^2}{(1-\gamma)^4\cdot N} .
	\end{eqnarray}
	\textbf{Step 3.} Find the bound for the final mean squared error  $\EE[\|g_j^i - \nabla_{\theta} F(\lambda(\theta_j^i))\|^2]$. 
	
	Compared to Step 1 and 2, Step 3 is relatively simple so we place the proof of Step 1 and Step 2 to Appendix \ref{appdx:SIVR-PG-variance-stp1} and \ref{appdx:SIVR-PG-variance-stp2} respectively. Given these inequalities, we have 
	\begin{eqnarray}
		\label{lm:SIVR-PG-variance-stp3-1}
		&&\EE\left[\big\|g_j^i-\nabla_{\theta} F(\lambda(\theta_{j}^i))\big\|^2\right]\nonumber\\
		& = & \EE\Big[\big\|g_j^i-\left[\nabla_\theta\lambda_H(\theta_j^i)\right]^\top r_{j-1}^i + \left[\nabla_\theta\lambda_H(\theta_j^i)\right]^\top r_{j-1}^i - \left[\nabla_\theta\lambda_H(\theta_j^i)\right]^\top r_{j}^i \nonumber\\
		& & + \left[\nabla_\theta\lambda_H(\theta_j^i)\right]^\top r_{j}^i - \nabla_{\theta} F(\lambda_H(\theta_{j}^i)) + \nabla_{\theta} F(\lambda_H(\theta_{j}^i)) -\nabla_{\theta} F(\lambda(\theta_{j}^i))\big\|^2\Big]\nonumber\\
		&\leq & 4\EE\left[\big\|g_j^i-\left[\nabla_\theta\lambda_H(\theta_j^i)\right]^\top r_{j-1}^i\big\|^2\right] + 4\EE\bigg[\big\|\left[\nabla_\theta\lambda_H(\theta_j^i)\right]^\top r_{j-1}^i - \left[\nabla_\theta\lambda_H(\theta_j^i)\right]^\top r_{j}^i\big\|^2\bigg]\nonumber\\
		& &  + 4\EE\bigg[\big\|\left[\nabla_\theta\lambda_H(\theta_j^i)\right]^\top r_{j}^i - \nabla_{\theta} F(\lambda_H(\theta_{j}^i))\big\|^2\bigg] + 4\EE\bigg[\big\|\nabla_{\theta} F(\lambda_H(\theta_{j}^i)) -\nabla_{\theta} F(\lambda(\theta_{j}^i))\big\|^2\bigg] \nonumber
	\end{eqnarray}
	The first term is given by \eqref{lm:SIVR-PG-variance-step1} and \eqref{lm:SIVR-PG-variance-step2}. The second term is similar to the proof of  \eqref{lm:importance-2} and \eqref{lm:SIVR-PG-variance-pf-4}, which gives 
	\begin{eqnarray}
		&& \EE\Big[\big\|\left[\nabla_\theta\lambda_H(\theta_j^i)\right]^\top r_{j-1}^i - \left[\nabla_\theta\lambda_H(\theta_j^i)\right]^\top r_{j}^i\big\|^2\Big]\nonumber\\
		&\leq & \frac{4\ell_\psi^2}{(1-\gamma)^4}\cdot\EE\left[\|r_{j-1}^i-r_j^i\|^2_\infty\right]\nonumber\\
		& =  & \frac{4\ell_\psi^2}{(1-\gamma)^4}\cdot\EE\left[\|\nabla_\lambda F(\lambda_{j-1}^i)-\nabla_\lambda F(\lambda_j^i)\|^2_\infty\right]\nonumber\\
		&\leq & \frac{4\ell_\psi^2\cdot L_{\lambda}^2}{(1-\gamma)^4}\cdot\EE\left[\|\lambda_{j-1}^i-\lambda_j^i\|^2\right]\nonumber\\
		& \leq &\frac{4\ell_\psi^2\cdot L_{\lambda}^2}{(1-\gamma)^4}\cdot\frac{2H(8\ell_\psi^2+L_\psi)(e^{2H\ell_\psi\delta} + 1)}{(1-\gamma)^3}\cdot\EE\left[\|\theta_{j}^i-\theta_{j-1}^i\|^2\right]\nonumber.
	\end{eqnarray}
	For the third term, 
	\begin{eqnarray}
		& & \EE\Big[\big\|\left[\nabla_\theta\lambda_H(\theta_j^i)\right]^\top r_{j}^i - \nabla_{\theta} F(\lambda_H(\theta_{j}^i))\big\|^2\Big]\nonumber\\
		& = & \EE\Big[\big\|\left[\nabla_\theta\lambda_H(\theta_j^i)\right]^\top \nabla_\lambda F(\lambda_{j}^i) - \left[\nabla_\theta\lambda_H(\theta_j^i)\right]^\top\nabla_{\lambda} F(\lambda_H(\theta_{j}^i))\big\|^2\Big]\nonumber\\
		& \leq &\frac{4\ell_\psi^2\cdot L_{\lambda}^2}{(1-\gamma)^4}\cdot\EE\big[\|\lambda_j^i -  \lambda_H(\theta_j^i)\|^2\big]\nonumber
	\end{eqnarray}
	where $\EE\big[\|\lambda_j^i -  \lambda_H(\theta_j^i)\|^2\big]$ is given by Lemma \ref{lemma:SIVR-occupancy}. For the last term, Lemma \ref{lemma:Trun-Occupancy} indicates that 
	\begin{equation*} 
		\big\|\nabla_{\theta} F(\lambda_H(\theta_{j}^i)) \!-\!\nabla_{\theta} F(\lambda(\theta_{j}^i))\big\|^2\!\leq\! \left(\frac{8\ell_\psi^2\cdot L_{\lambda}^2}{(1-\gamma)^6} + 16\ell_\psi^2\ell_{\lambda,\infty}^2\Big(\frac{(H+1)^2}{(1-\gamma)^2} + \frac{1}{(1-\gamma)^4}\Big) \right)\cdot\gamma^{2H}.
	\end{equation*}
	Combining all the above inequalities, we have the final result:
	\begin{equation*}
		\EE\left[\Big\|g_j^i-\nabla_{\theta} F(\lambda(\theta_{j}^i))\Big\|^2\right] \leq \frac{\delta_1}{N} + \delta_2\gamma^{2H} + \frac{\delta_3}{B}\cdot\sum_{j'=1}^j\EE\left[\|\theta_{j'-1}^i-\theta_{j'}^i\|^2\right] + \delta_4\EE\left[\|\theta_{j-1}^i-\theta_{j}^i\|^2\right]
	\end{equation*}
	where the $\delta_i$'s are defined according to the lemma. 
\end{proof}

\subsection{Proof of Step 1}
\label{appdx:SIVR-PG-variance-stp1}
\begin{proof}
	Consider the $j$-th step of the $i$-th epoch where the trajectories $\tau\in\cB_j^i$ are sampled under policy $\pi_{\theta_j^i}$. Similar to the analysis of Lemma \ref{lemma:SIVR-occupancy}, we have
	\begin{eqnarray}
		\label{lm:SIVR-PG-variance-pf-1}
		&&\EE\left[\Big\|g_j^i-\left[\nabla_\theta\lambda_H(\theta_j^i)\right]^\top r_{j-1}^i\Big\|^2\right]- \EE\left[\Big\|g_{j-1}^i-\left[\nabla_\theta\lambda_H(\theta_{j-1}^i)\right]^\top r_{j-2}^i\Big\|^2\right]\\
		& = & \EE\bigg[\Big\|\frac{1}{B}   \sum_{\tau\in\cB_j^i}\left(\widehat{g}(\tau|\theta_{j}^i,r_{j-1}^i) -  \widehat{g}_\omega(\tau|\theta_j^i,\theta_{j-1}^i,r_{j-2}^i) \right) + g_{j-1}^i-\left[\nabla_\theta\lambda_H(\theta_{j-1}^i)\right]^\top r_{j-2}^i\nonumber\\
		& & + \left[\nabla_\theta\lambda_H(\theta_{j-1}^i)\right]^\top r_{j-2}^i -\left[\nabla_\theta\lambda_H(\theta_j^i)\right]^\top r_{j-1}^i\Big\|^2\bigg]- \EE\left[\Big\|g_{j-1}^i-\left[\nabla_\theta\lambda_H(\theta_{j-1}^i)\right]^\top r_{j-2}^i\Big\|^2\right]\nonumber\\
		& = & \EE\bigg[\Big\|\frac{1}{B}\sum_{\tau\in\cB_j^i}\left(\widehat{g}(\tau|\theta_{j}^i,r_{j-1}^i) -  \widehat{g}_\omega(\tau|\theta_j^i,\theta_{j-1}^i,r_{j-2}^i)\right) + \left[\nabla_\theta\lambda_H(\theta_{j-1}^i)\right]^\top r_{j-2}^i -\left[\nabla_\theta\lambda_H(\theta_j^i)\right]^\top r_{j-1}^i \Big\|^2\bigg]\nonumber\\
		&& + 2\EE\bigg[\Big\langle \frac{1}{B}\sum_{\tau\in\cB_j^i}\left(\widehat{g}(\tau|\theta_{j}^i,r_{j-1}^i) -  \widehat{g}_\omega(\tau|\theta_j^i,\theta_{j-1}^i,r_{j-2}^i) \right) + \left[\nabla_\theta\lambda_H(\theta_{j-1}^i)\right]^\top r_{j-2}^i \nonumber\\
		&&-\left[\nabla_\theta\lambda_H(\theta_j^i)\right]^\top r_{j-1}^i\,\,,\,\, g_{j-1}^i-\left[\nabla_\theta\lambda_H(\theta_{j-1}^i)\right]^\top r_{j-2}^i\Big\rangle\bigg]\nonumber.
	\end{eqnarray}
	Let $\cF_{j-1}^{i}$ be the sigma algebra generated by the randomness until (including) the trajectory batch $\cB_{j}^i$. Then we have 
	\begin{eqnarray}
		&&\EE\bigg[\Big\langle \frac{1}{B}\sum_{\tau\in\cB_j^i} \widehat{g}(\tau|\theta_{j}^i,r_{j-1}^i) -\left[\nabla_\theta\lambda_H(\theta_j^i)\right]^\top r_{j-1}^i\,\,,\,\, g_{j-1}^i-\left[\nabla_\theta\lambda_H(\theta_{j-1}^i)\right]^\top r_{j-2}^i\Big\rangle\,\Big| \cF_{j-1}^i\bigg]\nonumber\\
		& = & \Big\langle \frac{1}{B}\sum_{\tau\in\cB_j^i} \EE\Big[\widehat{g}(\tau|\theta_{j}^i,r_{j-1}^i)\,\big| \cF_{j-1}^i\Big] -\left[\nabla_\theta\lambda_H(\theta_j^i)\right]^\top r_{j-1}^i\,\,,\,\, g_{j-1}^i-\left[\nabla_\theta\lambda_H(\theta_{j-1}^i)\right]^\top r_{j-2}^i\Big\rangle\nonumber\\
		& = & 0.\nonumber
	\end{eqnarray}
	Similarly, we have 
	\begin{eqnarray} 
		&& \EE\bigg[\Big\langle -\frac{1}{B}\sum_{\tau\in\cB_j^i}  \widehat{g}_\omega(\tau|\theta_j^i,\theta_{j-1}^i,r_{j-2}^i) + \left[\nabla_\theta\lambda_H(\theta_{j-1}^i)\right]^\top r_{j-2}^i \,\,,\,\, g_{j-1}^i-\left[\nabla_\theta\lambda_H(\theta_{j-1}^i)\right]^\top r_{j-2}^i\Big\rangle\,\Big| \cF_{j-1}^i\bigg]\nonumber\\
		& = & \Big\langle -\frac{1}{B}\sum_{\tau\in\cB_j^i}\EE\Big[\widehat{g}_\omega(\tau|\theta_j^i,\theta_{j-1}^i,r_{j-2}^i)\,\big| \cF_{j-1}^i\Big] + \left[\nabla_\theta\lambda_H(\theta_{j-1}^i)\right]^\top r_{j-2}^i \,\,,\,\, g_{j-1}^i-\left[\nabla_\theta\lambda_H(\theta_{j-1}^i)\right]^\top r_{j-2}^i\Big\rangle\nonumber\\
		& = & 0.\nonumber
	\end{eqnarray}
	Substituting the above two inequalities into \eqref{lm:SIVR-PG-variance-pf-1} yields
	\begin{eqnarray}
		\label{lm:SIVR-PG-variance-pf-2}
		&&\EE\left[\Big\|g_j^i-\left[\nabla_\theta\lambda_H(\theta_j^i)\right]^\top r_{j-1}^i\Big\|^2\right]- \EE\left[\Big\|g_{j-1}^i-\left[\nabla_\theta\lambda_H(\theta_{j-1}^i)\right]^\top r_{j-2}^i\Big\|^2\right]\\
		& = & \EE\bigg[\Big\|\frac{1}{B}\sum_{\tau\in\cB_j^i}\left(\widehat{g}(\tau|\theta_{j}^i,r_{j-1}^i) -  \widehat{g}_\omega(\tau|\theta_j^i,\theta_{j-1}^i,r_{j-2}^i) \right) + \left[\nabla_\theta\lambda_H(\theta_{j-1}^i)\right]^\top r_{j-2}^i -\left[\nabla_\theta\lambda_H(\theta_j^i)\right]^\top r_{j-1}^i \Big\|^2\bigg]\nonumber\\
		& \leq & \frac{1}{B}\cdot\EE\Big[\big\|\widehat{g}(\tau|\theta_{j}^i,r_{j-1}^i) -  \widehat{g}_\omega(\tau|\theta_j^i,\theta_{j-1}^i,r_{j-2}^i) \big\|^2\Big]\nonumber\\
		& = & \frac{1}{B}\cdot\EE\Big[\big\|\widehat{g}(\tau|\theta_{j}^i,r_{j-1}^i) - \widehat{g}(\tau|\theta_{j-1}^i,r_{j-1}^i) + \widehat{g}(\tau|\theta_{j-1}^i,r_{j-1}^i) - \widehat{g}(\tau|\theta_{j-1}^i,r_{j-2}^i) \nonumber\\
		& & + \widehat{g}(\tau|\theta_{j-1}^i,r_{j-2}^i) -  \widehat{g}_\omega(\tau|\theta_j^i,\theta_{j-1}^i,r_{j-2}^i) \big\|^2\Big]\nonumber\\
		&\leq & \frac{3}{B}\cdot\EE\Big[\big\| \widehat{g}(\tau|\theta_{j}^i,r_{j-1}^i) - \widehat{g}(\tau|\theta_{j-1}^i,r_{j-1}^i)\big\|^2\Big] + \frac{3}{B}\cdot\EE\Big[\big\|\widehat{g}(\tau|\theta_{j-1}^i,r_{j-1}^i) - \widehat{g}(\tau|\theta_{j-1}^i,r_{j-2}^i)\big\|^2\Big]\nonumber\\
		&& + \frac{3}{B}\cdot\EE\Big[\big\|\widehat{g}(\tau|\theta_{j-1}^i,r_{j-2}^i) -  \widehat{g}_\omega(\tau|\theta_j^i,\theta_{j-1}^i,r_{j-2}^i)\big\|^2\Big]\nonumber
	\end{eqnarray}
	For the first term, we have
	\begin{eqnarray}
		\label{lm:SIVR-PG-variance-pf-3}
		&&\EE\Big[\big\|\widehat{g}(\tau|\theta_{j}^i,r_{j-1}^i) - \widehat{g}(\tau|\theta_{j-1}^i,r_{j-1}^i)\big\|^2\Big] \\
		&\overset{(a)}{\leq}& \frac{4(\ell_\psi^2+L_\psi)^2}{(1-\gamma)^4}\cdot\EE\left[\|r_{j-1}^i\|_\infty^2\cdot\|\theta_{j}^i-\theta_{j-1}^i\|^2\right] \nonumber\\
		&\overset{(b)}{\leq}& \frac{4(\ell_\psi^2+L_\psi)^2\cdot\ell_{\lambda,\infty}^2}{(1-\gamma)^4}\cdot\EE\left[\|\theta_{j}^i-\theta_{j-1}^i\|^2\right] \nonumber,
	\end{eqnarray}
	where (a) is due to Lemma \ref{lemma:Lipschitz-g}; and (b) is due to Assumption \ref{assumption:F-Lip} indicates that $\|r_{j-1}^i\|_\infty^2 = \|\nabla_{\lambda}F(\lambda_{j-1}^i)\|_\infty^2\leq\ell_{\lambda,\infty}^2$ w.p. 1.  Similarly, combining Lemma \ref{lemma:Lipschitz-g}, Assumption \ref{assumption:F-Lip} and the $\lambda_j^i$ update formula \eqref{defn:SIVR-PG-small-1} yields
	\begin{eqnarray}
		\label{lm:SIVR-PG-variance-pf-4}
		& & \EE\bigg[\Big\|\widehat{g}(\tau|\theta_{j-1}^i,r_{j-1}^i) - \widehat{g}(\tau|\theta_{j-1}^i,r_{j-2}^i)\Big\|^2\bigg]\\
		& \leq & \frac{4\ell_\psi^2}{(1-\gamma)^4}\cdot\EE\left[\|r_{j-1}^i-r_{j-2}^i\|^2_\infty\right]\nonumber\\
		& = & \frac{4\ell_\psi^2}{(1-\gamma)^4}\cdot\EE\left[\|\nabla_\lambda F(\lambda_{j-1}^i)-\nabla_\lambda F(\lambda_{j-2}^i)\|^2_\infty\right]\nonumber\\
		& \leq & \frac{4\ell_\psi^2\cdot L_{\lambda}^2}{(1-\gamma)^4}\cdot\EE\left[\| \lambda_{j-1}^i- \lambda_{j-2}^i\|^2\right]\nonumber\\
		& = & \frac{4\ell_\psi^2\cdot L_{\lambda}^2}{(1-\gamma)^4}\cdot\EE\bigg[\Big\| \frac{1}{B}\sum_{\tau\in\cB_{j-1}^i}\left(\widehat\lambda(\tau|\theta_{j-1}^i) - \widehat\lambda_\omega\left(\tau|\theta_{j-1}^i,\theta_{j-2}^i\right) \right)\Big\|^2\bigg]\nonumber\\ 
		& \leq & \frac{4\ell_\psi^2\cdot L_{\lambda}^2}{(1-\gamma)^4}\cdot\EE\left[\big\| \widehat\lambda(\tau|\theta_{j-1}^i) - \widehat\lambda_\omega\left(\tau|\theta_{j-1}^i,\theta_{j-2}^i\right) \big\|^2\right] \nonumber\\
		& \leq & \frac{4\ell_\psi^2\cdot L_{\lambda}^2}{(1-\gamma)^4}\cdot\frac{2H(8\ell_\psi^2+L_\psi)(e^{2H\ell_\psi\delta} + 1)}{(1-\gamma)^3}\cdot\EE\left[\|\theta_{j-1}^i-\theta_{j-2}^i\|^2\right]\nonumber 
	\end{eqnarray}
	where the last inequality is due to the analysis of \eqref{lm:SIVR-variance-pf-2}. When $j = 1$, $r_{j-2}^i = r_{j-1}^i$ by default, and the above term is zero. To be compatible, we default $\theta_{-1}^i := \theta_{0}^i$. For the last term of \eqref{lm:SIVR-PG-variance-pf-2}, where the trajectory $\tau$ is sampled under the behavioral policy $\pi_{\theta_{j}^i}$, we have
	\begin{eqnarray}
		\label{lm:SIVR-PG-variance-pf-5}
		& & \EE\Big[\big\| \widehat{g}(\tau|\theta_{j-1}^i,r_{j-2}^i) -  \widehat{g}_\omega(\tau|\theta_j^i,\theta_{j-1}^i,r_{j-2}^i) \big\|^2\Big]\\
		& = & \EE\bigg[\Big\|\sum_{t=0}^{H-1}\gamma^t\cdot\big(1-\omega_t(\tau|\theta_j^i,\theta_{j-1}^i)\big)\cdot r_{j-2}^i(s_t,a_t)\cdot\Big(\sum_{t'=0}^t\nabla_{\theta}\log \pi_{\theta_{j-1}^i}(a_{t'}|s_{t'})\Big) \Big\|^2\bigg]\nonumber\\
		& \leq & 4H\ell_\psi^2\ell_{\lambda,\infty}^2\cdot\sum_{t=0}^{H-1}\gamma^t\cdot(t+1)^2\cdot\text{Var}\big(\omega_t(\tau|\theta_j^i,\theta_{j-1}^i)\big)\nonumber\\
		& \leq &  4H\ell_\psi^2\ell_{\lambda,\infty}^2\cdot\EE\left[\|\theta_{j-1}^i-\theta_{j}^i\|^2\right]\cdot\sum_{t=0}^{H-1}\gamma^t\cdot(t+1)^3\cdot(t+2)\big(8\ell_\psi^2+ L_\psi\big)(e^{2H\ell_\psi\delta} + 1)\nonumber\\
		&\leq& \frac{96H\ell_\psi^2\ell_{\lambda,\infty}^2\big(8\ell_\psi^2+ L_\psi\big)(e^{2H\ell_\psi\delta} + 1)}{(1-\gamma)^5}\cdot\EE\left[\|\theta_{j-1}^i-\theta_{j}^i\|^2\right]\nonumber.
	\end{eqnarray}
	Substituting \eqref{lm:SIVR-PG-variance-pf-3}, \eqref{lm:SIVR-PG-variance-pf-4} and \eqref{lm:SIVR-PG-variance-pf-5} into \eqref{lm:SIVR-PG-variance-pf-2} yields 
	\begin{eqnarray*} 
		&&\EE\left[\Big\|g_j^i-\left[\nabla_\theta\lambda_H(\theta_j^i)\right]^\top r_{j-1}^i\Big\|^2\right]- \EE\left[\Big\|g_{j-1}^i-\left[\nabla_\theta\lambda_H(\theta_{j-1}^i)\right]^\top r_{j-2}^i\Big\|^2\right] \\
		&\leq& \left(\frac{12(\ell_\psi+L_\psi)^2\cdot\ell_{\lambda,\infty}^2}{(1-\gamma)^4}+\frac{288H\ell_\psi^2\ell_{\lambda,\infty}^2\big(8\ell_\psi^2+ L_\psi\big)(e^{2H\ell_\psi\delta} + 1)}{(1-\gamma)^5}\right)\cdot\frac{\EE\left[\|\theta_{j}^i-\theta_{j-1}^i\|^2\right]}{B} \\
		&&+ \frac{24H\ell_\psi^2 L_{\lambda}^2(8\ell_\psi^2+L_\psi)(e^{2H\ell_\psi\delta} + 1)}{(1-\gamma)^7}\cdot\frac{\EE\left[\|\theta_{j-1}^i-\theta_{j-2}^i\|^2\right]}{B}
	\end{eqnarray*}
	Summing up the above inequality over $j$ proves the result. 
\end{proof}

\subsection{Proof of Step 2}
\label{appdx:SIVR-PG-variance-stp2}
\begin{proof}
	Define $(r_0^i)^* = \nabla_\lambda F(\lambda_H(\theta_0^i))$, for the ease of notation. Then 
	\begin{eqnarray}
		\label{lm:SIVR-PG-variance-pf-stp2-1}
		&&\EE\left[\big\|g_0^i-\left[\nabla_\theta\lambda_H(\theta_{0}^i)\right]^\top r_{0}^i\big\|^2\right]\\
		& = & \EE\bigg[\Big\|\frac{1}{N}\sum_{\tau\in\cN_i} \widehat{g}(\tau|\theta_{0}^i,r_0^i) - \frac{1}{N}\sum_{\tau\in\cN_i} \widehat{g}(\tau|\theta_{0}^i,(r_0^i)^*) + \frac{1}{N}\sum_{\tau\in\cN_i} \widehat{g}(\tau|\theta_{0}^i,(r_0^i)^*)-\nabla_{\theta} F(\lambda_H(\theta_{0}^i))\nonumber\\
		&& + \nabla_{\theta} F(\lambda_H(\theta_{0}^i)) - \left[\nabla_\theta\lambda_H(\theta_{0}^i)\right]^\top r_{0}^i\Big\|^2\bigg] \nonumber\\
		& \leq & 3\EE\bigg[\Big\|\frac{1}{N}\sum_{\tau\in\cN_i} \Big(\widehat{g}(\tau|\theta_{0}^i,r_0^i) - \widehat{g}(\tau|\theta_{0}^i,(r_0^i)^*)\Big)\Big\|^2\bigg] +3\EE\bigg[\Big\|\nabla_{\theta} F(\lambda_H(\theta_{0}^i)) - \left[\nabla_\theta\lambda_H(\theta_{0}^i)\right]^\top r_{0}^i\Big\|^2\bigg]  \nonumber\\
		& & + 3\EE\bigg[\Big\|\frac{1}{N}\sum_{\tau\in\cN_i}  \widehat{g}(\tau|\theta_{0}^i,(r_0^i)^*)-\nabla_{\theta} F(\lambda_H(\theta_{0}^i))\Big\|^2\bigg] \nonumber.
	\end{eqnarray}
	For the first term of \eqref{lm:SIVR-PG-variance-pf-stp2-1}, we have 
	\begin{eqnarray*}
		&&\EE\bigg[\Big\|\frac{1}{N}\sum_{\tau\in\cN_i} \Big(\widehat{g}(\tau|\theta_{0}^i,r_0^i) -  \widehat{g}(\tau|\theta_{0}^i,(r_0^i)^*)\Big)\Big\|^2\bigg]\\
		&\leq&\frac{1}{N}\sum_{\tau\in\cN_i}\EE\bigg[\Big\|\widehat{g}(\tau|\theta_{0}^i,r_0^i) -  \widehat{g}(\tau|\theta_{0}^i,(r_0^i)^*)\Big\|^2\bigg]\\
		&\overset{(a)}{\leq}& \frac{4\ell_\psi^2}{(1-\gamma)^4}\cdot\EE\left[\|r_0^i - (r_0^i)^*\|^2_\infty\right] \\
		& \overset{(b)}{\leq} & \frac{4\ell_\psi^2\cdot L_{\lambda}^2}{(1-\gamma)^4}\cdot\EE\left[\|\lambda_0^i - \lambda_H(\theta_0^i)\|^2\right]\\
		& \overset{(c)}{\leq} & \frac{4\ell_\psi^2\cdot L_{\lambda}^2}{N(1-\gamma)^6}
	\end{eqnarray*}
	where (a) is due to Lemma \ref{lemma:Lipschitz-g}; (b) is because $r_0^i = \nabla_\lambda F(\lambda_0^i)$, $(r_0^i)^* = \nabla_\lambda F(\lambda_H(\theta_0^i))$ and Assumption \ref{assumption:F-Lip}; (c) is because Lemma \ref{lemma:SIVR-occupancy}. Similarly, we can show that 
	$$\EE\bigg[\Big\|\frac{1}{N}\sum_{\tau\in\cN_i} \widehat{g}(\tau|\theta_{0}^i,(r_0^i)^*)-\nabla_{\theta} F(\lambda_H(\theta_{0}^i))\Big\|^2\bigg]\leq \frac{\ell_{\lambda,\infty}^2}{N(1-\gamma)^4}$$
	and 
	$$\EE\bigg[\Big\|\nabla_{\theta} F(\lambda_H(\theta_{0}^i)) - \left[\nabla_\theta\lambda_H(\theta_{0}^i)\right]^\top r_{0}^i\Big\|^2\bigg]\leq \frac{4\ell_\psi^2\cdot L_{\lambda}^2}{N(1-\gamma)^6}.$$
	Substituting the above three bounds into \eqref{lm:SIVR-PG-variance-pf-stp2-1} proves the inequality \eqref{lm:SIVR-PG-variance-step2}.
\end{proof}

\section{Proof of Theorem \ref{theorem:ncvx}}
\begin{proof}
	Summing up the ascent inequality of Lemma \ref{lemma:ascent-ncvx} for the $i$-th epoch and taking the expectation on both sides, we have 
	\begin{eqnarray}
		&&\frac{\eta}{4}\sum_{j=0}^{m-1}\EE\big[\|\cG_\eta(\theta_j^i)\|^2\big]\nonumber\\ &\overset{(i)}{\leq}&
		\EE\big[F(\lambda(\theta_m^i))\big]- \EE\big[F(\lambda(\theta_0^i))\big] - \Big(\frac{1}{2\eta}-L_\theta\Big)\cdot\sum_{j=0}^{m-1}\EE\big[\|\theta_{j+1}^i-\theta_j^i\|^2\big] \nonumber\\
		&&+ \Big(\frac{\eta}{2}+ \frac{1}{2L_\theta}\Big)\cdot\sum_{j=0}^{m-1} \EE\big[\|\nabla_\theta F(\lambda(\theta_j^i))-g_j^i\|^2\big]\nonumber\\
		&\overset{(ii)}{\leq}& \EE\big[F(\lambda(\theta_m^i))\big]- \EE\big[F(\lambda(\theta_0^i))\big] + m\Big(\frac{\eta}{2}+ \frac{1}{2L_\theta}\Big)\Big(\frac{C_1}{N} + \gamma^{2H}C_2\Big)\nonumber\\
		&& - \left(\Big(\frac{1}{2\eta}-L_\theta\Big) - \Big(\frac{\eta}{2}+ \frac{1}{2L_\theta}\Big)\cdot\Big(\frac{m}{B}C_3 + C_4 \Big)\right)\cdot\sum_{j=0}^{m-1} \EE\big[\|\nabla_\theta F(\lambda(\theta_j^i))-g_j^i\|^2\big]\nonumber,
	\end{eqnarray}
	where (i) is because Lemma \ref{lemma:ascent-ncvx} and (ii) is because Lemma \ref{lemma:SIVR-variance}.
	Note that $N = B^2 = m^2$, and $\eta=\frac{1}{1+(C_3+C_4)/L_\theta^{2}}\cdot\frac{1}{2L_\theta}$, the coefficient 
	$$\Big(\frac{1}{2\eta}-L_\theta\Big) - \Big(\frac{\eta}{2}+ \frac{1}{2L_\theta}\Big)\cdot\Big(\frac{m}{B}C_3 + C_4 \Big)\geq0.$$
	Hence we have 
	\begin{eqnarray}
		\frac{\eta}{4}\sum_{j=0}^{m-1}\EE\big[\|\cG_\eta(\theta_j^i)\|^2\big] 
		&\leq& \EE\big[F(\lambda(\theta_m^i))\big]- \EE\big[F(\lambda(\theta_0^i))\big] + m\Big(\frac{\eta}{2}+ \frac{1}{2L_\theta}\Big)\Big(\frac{C_1}{N} + \gamma^{2H}C_2\Big)\nonumber.
	\end{eqnarray}
	Note that $\lambda(\theta_0^i) = \tilde \theta_{i-1}$ and $\lambda(\theta_m^i) = \tilde \theta_i$, summing the above inequality over all $T$ epochs and dividing both sides with $\frac{\eta}{4}\cdot Tm$ yields
	\begin{eqnarray}
		\frac{1}{Tm}\sum_{j=0}^{m-1}\sum_{i =1}^T\EE\big[\|\cG_\eta(\theta_j^i)\|^2\big] 
		&\leq& \frac{4(F(\lambda(\theta^*))- F(\lambda(\theta_0^i)))}{Tm\cdot\eta} + \Big(2+ \frac{2}{L_\theta\eta}\Big)\Big(\frac{C_1}{N} + \gamma^{2H}C_2\Big)\nonumber.
	\end{eqnarray}
	Choosing $T = m = \epsilon^{-1}$, $H = \frac{2\log(\epsilon^{-1})}{1-\gamma}$ and let $\theta_{out}$ be selected uniformly at random from $\{\theta_{j}^i: i = 1,...,T, j = 0,...,m-1\}$ yields
	\begin{eqnarray}
		&&\EE\big[\|\cG_\eta(\theta_{out})\|^2\big]\nonumber\\
		& = &\frac{1}{Tm}\sum_{j=0}^{m-1}\sum_{i =1}^T\EE\big[\|\cG_\eta(\theta_j^i)\|^2\big] \nonumber\\
		&\leq& \frac{4(F(\lambda(\theta^*))- F(\lambda(\theta_0^i)))}{Tm\cdot\eta} + \Big(2+ \frac{2}{L_\theta\eta}\Big)\Big(\frac{C_1}{N} + \gamma^{2H}C_2\Big)\nonumber\\
		& = & \left(4\eta^{-1}\cdot(F(\lambda(\theta^*))- F(\lambda(\theta_0^i))) + (6+(C_3+C_4)/L_\theta^{2})\cdot(C_3 + \gamma^{2H}\epsilon^{-2}C_4)\right)\cdot\epsilon^2\nonumber\\
		& = & \cO(\epsilon^2).\nonumber
	\end{eqnarray}
	Then Lemma \ref{lemma:optimality-measure} indicates that $\EE\big[\|\nabla_\theta F(\lambda(\theta_{out}))\|^2\big]\leq \cO(\epsilon^{2})$. Then by Jensen's inequality, 
	$$\EE\big[\|\nabla_\theta F(\lambda(\theta_{out}))\|\big]\leq\sqrt{\EE\big[\|\nabla_\theta F(\lambda(\theta_{out}))\|^2\big]}\leq \cO(\epsilon).$$
	Hence we complete the proof. 
\end{proof}

\section{Proof of Lemma \ref{lemma:ascent-0}}
\label{appdx:ascent-0}
\begin{proof}
	By Lemma \ref{lemma:importance}, the function $F\circ\lambda(\cdot)$ is $L_\theta$-smooth. Then we have  
	\begin{eqnarray*}
		&&|F(\lambda(\theta)) - F(\lambda(\theta_{j}^i)) - \langle g_j^i, \theta-\theta_{j}^i\rangle|\\
		&\leq & |F(\lambda(\theta)) - F(\lambda(\theta_{j}^i)) - \langle \nabla_\theta F(\lambda(\theta_{j}^i)), \theta-\theta_{j}^i\rangle| + |\langle \nabla_\theta F(\lambda(\theta_{j}^i))-g_j^i, \theta-\theta_{j}^i\rangle|\\
		& \leq & \frac{L_\theta}{2}\|\theta - \theta_j^i\|^2 + \frac{L_\theta}{2}\|\theta - \theta_j^i\|^2 + \frac{1}{2L_\theta}\|\nabla_\theta F(\lambda(\theta_{j}^i))-g_j^i\|^2\\
		& = &L_\theta\|\theta - \theta_j^i\|^2 +  \frac{1}{2L_\theta}\|\nabla_\theta F(\lambda(\theta_{j}^i))-g_j^i\|^2.
	\end{eqnarray*}
	That is, 
	\begin{eqnarray}
		\label{lm:ascent-pf-1}
		&&F(\lambda(\theta_{j}^i)) + \langle g_j^i, \theta-\theta_{j}^i\rangle + L_\theta\|\theta-\theta_{j}^i\|^2 + \frac{1}{2L_\theta}\|\nabla_\theta F(\lambda(\theta_{j}^i))-g_j^i\|^2\\
		& \geq & 
		F(\lambda(\theta))\nonumber\\
		&\geq& F(\lambda(\theta_{j}^i)) + \langle g_j^i, \theta-\theta_{j}^i\rangle - L_\theta\|\theta-\theta_{j}^i\|^2 - \frac{1}{2L_\theta}\|\nabla_\theta F(\lambda(\theta_{j}^i))-g_j^i\|^2.\nonumber
	\end{eqnarray} 
	By the discussion of \eqref{defn:TR}, our truncated gradient update is equivalent to solving
	\begin{eqnarray}
		\label{lm:ascent-pf-2}
		\theta_{j+1}^i:=\argmax_{\|\theta-\theta_j^i\|\leq\delta} F(\lambda(\theta_{j}^i)) + \langle g_j^i,\theta-\theta_j^i\rangle - \frac{1}{2\eta}\|\theta-\theta_{j}^i\|^2.
	\end{eqnarray}
	By \eqref{lm:ascent-pf-1} and \eqref{lm:ascent-pf-2}, we have 
	\begin{eqnarray} 
		&&F(\lambda(\theta_{j+1}^i))\nonumber \\
		& \overset{(a)}{\geq} & F(\lambda(\theta_{j}^i)) + \langle g_j^i, \theta_{j+1}^i-\theta_{j}^i\rangle - L_\theta\|\theta_{j+1}^i-\theta_{j}^i\|^2 - \frac{1}{2L_\theta}\|\nabla_\theta F(\lambda(\theta_{j}^i))-g_j^i\|^2\nonumber\\
		& = & F(\lambda(\theta_{j}^i)) + \langle g_j^i, \theta_{j+1}^i-\theta_{j}^i\rangle - \frac{1}{2\eta}\|\theta_{j+1}^i-\theta_{j}^i\|^2 - \frac{1}{2L_\theta}\|\nabla_\theta F(\lambda(\theta_{j}^i))-g_j^i\|^2  \nonumber\\
		&& + \Big(\frac{1}{2\eta}-L_\theta\Big)\cdot\|\theta_{j+1}^i-\theta_{j}^i\|^2\nonumber\\
		& \overset{(b)}{=} & \max_{\|\theta-\theta_j^i\|\leq\delta}\left\{F(\lambda(\theta_{j}^i)) + \langle g_j^i, \theta-\theta_{j}^i\rangle - \frac{1}{2\eta}\|\theta-\theta_{j}^i\|^2\right\} - \frac{1}{2L_\theta}\|\nabla_\theta F(\lambda(\theta_{j}^i))-g_j^i\|^2\nonumber\\
		&&+ \Big(\frac{1}{2\eta}-L_\theta\Big)\cdot\|\theta_{j+1}^i-\theta_{j}^i\|^2\nonumber\\
		& \overset{(c)}{\geq}& \max_{\|\theta-\theta_j^i\|\leq\delta}\left\{F(\lambda(\theta)) - \Big(\frac{1}{2\eta}+L_\theta\Big)\|\theta-\theta_{j}^i\|^2\right\} - \frac{1}{L_\theta}\|\nabla_\theta F(\lambda(\theta_{j}^i))-g_j^i\|^2 \nonumber\\
		&&+ \Big(\frac{1}{2\eta}-L_\theta\Big)\cdot\|\theta_{j+1}^i-\theta_{j}^i\|^2\nonumber,
	\end{eqnarray}
	where (a) is by setting $\theta = \theta_{j+1}^i$ in second half of \eqref{lm:ascent-pf-1}, (b) is due to  \eqref{lm:ascent-pf-2}, and (c) is due to the first half of \eqref{lm:ascent-pf-1}. That is, 
	\begin{align}
		\label{lm:ascent-pf-4}
		F(\lambda(\theta^i_{j+1})) &\,\,\geq\,\, \max_{\|\theta-\theta^i_j\|\leq\delta} \left\{F(\lambda(\theta)) - \Big(L_\theta + \frac{1}{2\eta}\Big)\cdot\|\theta - \theta_j^i\|^2\right\} \\
		& \qquad\qquad\qquad+ \Big(\frac{1}{2\eta}-L_\theta\Big)\cdot\|\theta_{j+1}^i - \theta_j^i\|^2 - \frac{1}{L_\theta}\|g_j^i - \nabla_\theta F(\lambda(\theta_j^i))\|^2.\nonumber
	\end{align} 
	For any $\epsilon<\bar\epsilon$, by Assumption \ref{assumption:over-para}, $(1-\epsilon)\lambda(\theta_j^i) + \epsilon\lambda(\theta^*)\in\mathcal{V}_{\lambda(\theta_j^i)}$ and hence $$\theta_\epsilon := \big(\lambda|_{\mathcal{U}_{\theta_j^i}}\big)^{-1}\big((1-\epsilon)\lambda(\theta_j^i) + \epsilon\lambda(\theta^*)\big)\in\mathcal{U}_{\theta_j^i}\subset B(\theta_j^i,\delta).$$
	Consequently, substituting the above $\theta_\epsilon$ into \eqref{lm:ascent-pf-4} yields
	\begin{eqnarray} 
		\label{thm:cvg-rate-pf-1}
		F(\lambda(\theta^i_{j+1})) &\geq & F\circ\lambda\circ\big(\lambda|_{\mathcal{U}_{\theta_j^i}}\big)^{-1}\big((1-\epsilon)\lambda(\theta_j^i) + \epsilon\lambda(\theta^*)\big) \\
		& & \!\!\!\!- \Big(L_\theta + \frac{1}{2\eta}\Big)\cdot\|\psi\left((1-\epsilon)\cdot\lambda(\theta_{j}^i) + \epsilon\cdot\lambda^*\right) - \theta_j^i\|^2  \nonumber\\
		&& \!\!\!\!+ \Big(\frac{1}{2\eta}-L_\theta\Big)\cdot\|\theta_{j+1}^i - \theta_j^i\|^2 - \frac{1}{L_\theta}\|g_j^i - \nabla_\theta F(\lambda(\theta_j^i))\|^2\nonumber.
	\end{eqnarray}
	Note that $\lambda\circ\big(\lambda|_{\mathcal{U}_{\theta_j^i}}\big)^{-1}$ is the identity mapping on $\mathcal{U}_{\theta_j^i}$, then 
	\begin{eqnarray*}
		F\circ\lambda\circ\big(\lambda|_{\mathcal{U}_{\theta_j^i}}\big)^{-1}\big((1-\epsilon)\lambda(\theta_j^i) + \epsilon\lambda(\theta^*)\big)  & = & F\left((1-\epsilon)\cdot\lambda(\theta_{j}^i) + \epsilon\cdot\lambda(\theta^*)\right)\\
		& \geq &  (1-\epsilon)\cdot F(\lambda(\theta_{j}^i)) + \epsilon \cdot F(\lambda(\theta^*)) 
	\end{eqnarray*}
	and 
	\begin{eqnarray*}
		& &\big\|\big(\lambda|_{\mathcal{U}_{\theta_j^i}}\big)^{-1}\big((1-\epsilon)\lambda(\theta_j^i) + \epsilon\lambda(\theta^*)\big) - \theta_j^i\big\|^2\\
		& = & \big\|\big(\lambda|_{\mathcal{U}_{\theta_j^i}}\big)^{-1}\big((1-\epsilon)\lambda(\theta_j^i) + \epsilon\lambda(\theta^*)\big) - \big(\lambda|_{\mathcal{U}_{\theta_j^i}}\big)^{-1}(\lambda(\theta_j^i))\big\|^2\nonumber\\
		& \overset{(a)}{\leq} & \epsilon^2\ell_{\theta}^2\cdot\|\lambda(\theta_{j}^i) - \lambda^*\|^2\\
		& = & \epsilon^2\ell_{\theta}^2\cdot\left(\|\lambda(\theta_{j}^i)\|^2 +  \|\lambda^*\|^2 - 2\langle\lambda(\theta_j^i),\lambda^*\rangle\right)\\
		& \leq& \frac{2\epsilon^2\ell_{\theta}^2}{(1-\gamma)^2} \nonumber
	\end{eqnarray*}
	where (a) is due to Assumption \ref{assumption:over-para}.
	Substituting the above two inequalities into \eqref{thm:cvg-rate-pf-1} and slightly rearranging the terms yields 
	\begin{align}  
		&F(\lambda(\theta^*))-F(\lambda(\theta^i_{j+1})) \,\,\leq\,\,  (1-\epsilon)\left(F(\lambda(\theta^*))-F(\lambda(\theta_{j}^i))\right)  \nonumber\\ 
		& \qquad\qquad\qquad+ \Big(L_\theta+ \frac{1}{2\eta}\Big)\frac{2\epsilon^2\ell_{\theta}^2}{(1-\gamma)^2} - \Big(\frac{1}{2\eta}- L_\theta\Big)\|\theta_{j+1}^i - \theta_j^i\|^2 + \frac{1}{L_\theta}\|g_j^i - \nabla_\theta F(\lambda(\theta_j^i))\|^2,\nonumber
	\end{align}
	which completes the proof. 
\end{proof}

\section{Proof of Theorem \ref{theorem:convergence-rate}}
\label{appdx:thm}
\begin{proof}
	Recall Lemma \ref{lemma:ascent-0}, where we have 
	\begin{align} 
		\label{thm:cvg-rate-pf-2}
		F(\lambda(\theta^*))-F(\lambda(\theta^i_{j+1}))& \leq  (1-\epsilon)\left(F(\lambda(\theta^*))-F(\lambda(\theta_{j}^i))\right) + \Big(L_\theta+ \frac{1}{2\eta}\Big)\frac{2\epsilon^2\ell_{\theta}^2}{(1-\gamma)^2}  \nonumber\\
		&\quad -\Big(\frac{1}{2\eta}-L_\theta\Big)\|\theta_{j+1}^i - \theta_j^i\|^2 +\frac{1}{L_\theta}\|g_j^i - \nabla_\theta F(\lambda(\theta_j^i))\|^2
	\end{align}
	For the ease of notation, let us denote $\sigma_{j}^i = - \Big(\frac{1}{2\eta}-L_\theta \Big)\cdot\|\theta_{j+1}^i - \theta_j^i\|^2 + \frac{1}{L_\theta}\|g_j^i - \nabla_\theta F(\lambda(\theta_j^i))\|^2$. Taking the expectation and telescoping \eqref{thm:cvg-rate-pf-2} over $j$ yields
	\begin{eqnarray}
		\label{thm:cvg-rate-pf-3}
		\EE\left[F(\lambda(\theta^*))\!-\!F(\lambda(\theta_{m}^i))\right] &\leq& (1\!-\!\epsilon)^m\cdot\EE\left[F(\lambda(\theta^*))\!-\!F(\lambda(\theta_{0}^i))\right] \!+\! \frac{(2L_\theta\!+\!\frac{1}{\eta})\ell_{\theta}^2}{(1-\gamma)^2}\!\cdot\!\epsilon\\
		&&\!+\! \sum_{j=0}^{m-1}(1\!-\!\epsilon)^{m-j-1}\EE[\sigma_j^i]\nonumber
	\end{eqnarray}
	Next, we show that when the step size $\eta$ is properly chosen, the term $\sum_{j=0}^{m-1}(1-\epsilon)^{m-j-1}\EE[\sigma_j^i]$ will be negligible.
	\begin{eqnarray*}
		& & \sum_{j=0}^{m-1}(1\!-\!\epsilon)^{m-j-1}\EE[\sigma_j^i]\\
		& = & \sum_{j=0}^{m-1}(1\!-\!\epsilon)^{m-j-1}\cdot\EE\left[- \Big(\frac{1}{2\eta}-L_\theta\Big)\cdot\|\theta_{j+1}^i - \theta_j^i\|^2 + \frac{1}{L_\theta}\|g_j^i - \nabla_\theta F(\lambda(\theta_j^i))\|^2\right]\nonumber\\
		& \leq & \frac{1}{L_\theta}\sum_{j=0}^{m-1}\EE\left[\|g_j^i - \nabla_\theta F(\lambda(\theta_j^i))\|^2\right] - (1-\epsilon)^m\Big(\frac{1}{2\eta}-L_\theta\Big)\cdot\sum_{j=0}^{m-1}\EE\left[\|\theta_{j+1}^i - \theta_j^i\|^2\right]\nonumber
	\end{eqnarray*}
	Let us choose $m = \epsilon^{-1}$. Then as long as $\epsilon\leq 1/2$, we have $(1-\epsilon)^{\epsilon^{-1}}\geq\frac{1}{4}$. Then, using Lemma \ref{lemma:SIVR-variance} yields
	\begin{eqnarray}
		& & \sum_{j=0}^{m-1}(1\!-\!\epsilon)^{m-j-1}\EE[\sigma_j^i]\nonumber\\
		& \leq & \frac{1}{L_\theta}\sum_{j=0}^{m-1}\EE\left[\|g_j^i - \nabla_\theta F(\lambda(\theta_j^i))\|^2\right] - \Big(\frac{1}{8\eta}-\frac{L_\theta}{4}\Big)\cdot\sum_{j=0}^{m-1}\EE\left[\|\theta_{j+1}^i - \theta_j^i\|^2\right]\nonumber\\
		& \leq & \frac{mC_1}{L_\theta N} + \frac{mC_2\cdot\gamma^{2H}}{L_\theta} + \frac{C_3}{L_\theta B}\cdot\sum_{j=0}^{m-1}\sum_{j'=0}^{j-1}\EE\left[\|\theta_{j'+1}^i-\theta_{j'}^i\|^2\right]- \Big(\frac{1}{8\eta}-\frac{L_\theta}{4} - \frac{C_4}{L_\theta}\Big)\cdot\sum_{j=0}^{m-1}\EE\left[\|\theta_{j+1}^i - \theta_j^i\|^2\right]\nonumber\\
		& \leq & \frac{mC_1}{L_\theta N} + \frac{mC_2\cdot\gamma^{2H}}{L_\theta} - \Big(\frac{1}{8\eta}-\frac{L_\theta}{4}- \frac{mC_3}{L_\theta B} - \frac{C_4}{L_\theta}\Big)\cdot\sum_{j=0}^{m-1}\EE\left[\|\theta_{j+1}^i - \theta_j^i\|^2\right]\nonumber
	\end{eqnarray}
	Since we choose $\eta \leq \frac{1}{2L_\theta + \frac{8(C_3+C_4)}{L_\theta}}$, then $\frac{1}{8\eta}-\frac{L_\theta}{4}- \frac{mC_3}{L_\theta B} - \frac{C_4}{L_\theta}$. Because we choose $B = m = \epsilon^{-1}$ and $N = \epsilon^{-2}$, then 
	\begin{eqnarray*} 
		\sum_{j=0}^{m-1}(1\!-\!\epsilon)^{m-j-1}\EE[\sigma_j^i] \leq \frac{C_1}{L_\theta}\epsilon + \frac{C_2}{L_\theta\epsilon}\cdot\gamma^{2H}.
	\end{eqnarray*} 
	Substituting the above inequality into \eqref{thm:cvg-rate-pf-3} and use the fact that $(1-\epsilon)^{\epsilon^{-1}}\leq \frac{1}{2}, \forall \epsilon\leq 1$ yields
	\begin{eqnarray*} 
		\EE\left[F(\lambda(\theta^*))\!-\!F(\lambda(\theta_{m}^i))\right] \leq \frac{1}{2}\EE\left[F(\lambda(\theta^*))\!-\!F(\lambda(\theta_{0}^i))\right] \!+\! \left(\frac{(2L_\theta\!+\!\frac{1}{\eta})\ell_{\theta}^2}{(1-\gamma)^2}+\frac{C_1}{L_\theta}\right)\cdot\epsilon + \frac{C_2}{L_\theta\epsilon}\cdot\gamma^{2H}.
	\end{eqnarray*}
	Using the fact that $\tilde\theta_{i-1} = \theta_{0}^i$ and $\tilde{\theta}_{i} = \theta_{m}^i$ proves 
	\begin{eqnarray*} 
		\EE\left[F(\lambda(\theta^*))\!-\!F(\lambda(\tilde{\theta}_{i}))\right] \leq \frac{1}{2}\EE\left[F(\lambda(\theta^*))\!-\!F(\lambda(\tilde\theta_{i-1}))\right] \!+\! \left(\frac{(2L_\theta\!+\!\frac{1}{\eta})\ell_{\theta}^2}{(1-\gamma)^2}+\frac{C_1}{L_\theta}\right)\!\cdot\!\epsilon + \frac{C_2}{L_\theta\epsilon}\cdot\gamma^{2H}.
	\end{eqnarray*}
	Again, telescoping sum the above inequality for $i=1,...,T$ proves 
	\begin{align*}  
		\EE\Big[F(\lambda(\theta^*))&-F(\lambda(\tilde\theta_{T}))\Big] \,\,\leq\,\, \left(F(\lambda(\theta^*))-F(\lambda(\tilde\theta_{0})) + \frac{(4L_\theta+\frac{2}{\eta})\ell_{\theta}^2}{(1-\gamma)^2}+\frac{2C_1}{L_\theta}\right)\cdot\frac{1}{2^T} + \frac{2C_2}{L_\theta\epsilon}\cdot\gamma^{2H},\nonumber
	\end{align*}
	Note that we choose $T = \log_2(\epsilon^{-1})$ and $H = \frac{2\log(1/\epsilon)}{1-\gamma}$, indicating that $\frac{1}{2^T} = \epsilon$ and $\gamma^{2H} = \cO(\epsilon^4)$. Therefore, 
	$\EE\big[F(\lambda(\theta^*))-F(\lambda(\tilde\theta_{T}))\big] \leq \cO(\epsilon)$.
\end{proof}

\section{Experiment Settings}
\subsection{Experiment Setting for Ordinary RL Tasks}
The snapshot batchsize $N$ is chosen from the grid search from $\{10, 25, 50, 100, 200\}$.  $B$ and $m$ are calculated according to the provided formula of the theory. The learning rate $\eta$ is also chosen from a grid search from the range $[10^{-4}, 10^{-1}]$. More details are presented in the list below. 

\begin{table}[htb!]
	\center
	\begin{tabular}{|c | c | c |c| c| c| c| c| c| c|}
		\hline
		Environment & Algorithm & \!\!Policy Network\!\! & $\gamma$ & $H$ &  $B$ &  $N$ & \!\!$m$\!\! & $\eta$ & $\delta$\\
		\hline
		& TSIVR-PG &&&& \!\!5\!\! & \!\!25\!\! & \!\!5\!\! &  \!\!$5\!\times\! 10^{-3}$\!\! & \!\!$1\!\times\!10^{-2}$\!\!\\
		& \!\!REINFORCE\!\!&&&&  & 25 &  &  $5\!\times\!10^{-3}$&\\
		CartPole-v0 & SVRPG & $4\!\times\!64\!\times\!64\!\times\!2$ & \!\!0.99\!\! & \!\!200\!\! &8 & 25 & 3 &  $5\!\times\!10^{-3}$&\\
		& SRVR-PG&&&& \!\!5\!\! & \!\!25\!\! & \!\!5\!\! &  \!\!$5\!\times\!10^{-3}$\!\!&\\
		& HSPGA&&&& \!\!5\!\! & \!\!25\!\! & \!\!5\!\! &  \!\!$8\!\times\!10^{-3}$\!\!&\\
		\hline
		& TSIVR-PG &&&& \!\!10\!\! & \!\!100\!\! & \!\!10\!\! &  \!\!$1\!\times\!10^{-1}$\!\! & \!\!$1\!\times\!10^{-2}$\!\!\\
		& \!\!REINFORCE\!\!&&&&  & \!\!100\!\! &  &  \!\!$5\!\times\!10^{-2}$\!\!&\\
		\!\!FrozenLake8x8\!\! & SVRPG & $64\!\times\!4$ & \!\!0.99\!\! & \!\!200\!\! &\!\!20\!\! & \!\!100\!\! & \!\!5\!\! &  \!\!$5\!\times\!10^{-2}$\!\!&\\
		& SRVR-PG&&&& \!\!10\!\! & \!\!100\!\! & \!\!10\!\! &  \!\!$5\!\times\!10^{-2}$\!\!&\\
		& HSPGA&&&& \!\!10\!\! &  \!\!100\!\! & \!\!10\!\! &  \!\!$8\!\times\!10^{-2}$\!\!&\\
		\hline
		& TSIVR-PG &&&& \!\!10\!\! & \!\!100\!\! & \!\!10\!\! &  \!\!$5\!\times\!10^{-3}$\!\! &\!\!$1\!\times\!10^{-2}$\!\!\\
		& \!\!REINFORCE\!\!&&&&  & \!\!100\!\! &  &  \!\!$2\!\times\!10^{-3}$\!\!&\\
		Acrobot-v1 & SVRPG & \!\!$4\!\times\!64\!\times\!64\!\times\!3$\!\! & \!\!0.999\!\! & \!\!500\!\! &\!\!20\!\! & \!\!100\!\! & \!\!5\!\! &  \!\!$2\!\times\!10^{-3}$\!\!&\\
		& SRVR-PG&&&& \!\!10\!\! & \!\!100\!\! & \!\!10\!\! &  \!\!$2\!\times\!10^{-3}$\!\!&\\
		& HSPGA&&&& \!\!10\!\! & \!\!100\!\! & \!\!10\!\! &  \!\!$2\!\times\!10^{-3}$\!\!&\\
		\hline
	\end{tabular}
\end{table}

\subsection{Experiment Setting for Maximizing Non-linear Objective Function}
The parameter selection is performed in the same way as the previous section. The inner planning loop of MaxEnt is performed by a single policy gradient step.  The details are presented in the list below. 
\begin{table}[htb!]
	\center
	\begin{tabular}{| c | c |c| c| c| c| c| c| c|}
		\hline
		Algorithm & Policy Network & $\gamma$ & $H$ &  $B$ &  $N$ &  $m$ & $\eta$ & $\delta$\\
		\hline
		TSIVR-PG & $64\times 4$ & 0.99 & 200 &20 & 100 & 5 &  $5\times 10^{-2}$& $9\times 10^{-2}$\\
		MaxEnt & $64\times 4$ &0.99&200&  & 100 &  &  $9\times 10^{-2}$ & \\
		\hline
	\end{tabular}
\end{table}

\end{document}